\documentclass{article}

\usepackage[final]{neurips_2024}

\usepackage[utf8]{inputenc} %
\usepackage[T1]{fontenc}    %
\usepackage{hyperref}       %
\usepackage{url}            %
\usepackage{booktabs}       %
\usepackage{amsfonts}       %
\usepackage{nicefrac}       %
\usepackage{microtype}      %
\usepackage{xcolor}         %

\usepackage{amsmath}
\usepackage{amssymb}
\usepackage{mathtools}
\usepackage{amsthm}

\usepackage{algorithm,algpseudocode}

\usepackage[disable,textsize=tiny]{todonotes}
\usepackage{graphicx}
\usepackage[caption = false]{subfig}
\usepackage{wrapfig}
\usepackage[font=small,labelfont=bf]{caption}
\usepackage{tabularray}

\usepackage[capitalize,noabbrev]{cleveref}

\newtheorem{theorem}{Theorem}[section]
\newtheorem{corollary}[theorem]{Corollary}
\newtheorem{proposition}[theorem]{Proposition}
\newtheorem{lemma}[theorem]{Lemma}
\newtheorem{definition}[theorem]{Definition}

\newcommand{\R}{\mathbb{R}}

\newcommand{\X}{\mathbf{X}}
\newcommand{\W}{\mathbf{W}}
\newcommand{\M}{\mathbf{M}}
\newcommand{\hh}{\hat{h}}
\newcommand{\Q}{\mathbf{Q}}
\newcommand{\K}{\mathbf{K}}
\newcommand{\V}{\mathbf{V}}
\newcommand{\XX}{\mathbf{X}}
\newcommand{\Ho}{\mathbf{H}}
\newcommand{\U}{\mathbf{U}}
\newcommand{\Lam}{\mathbf{\Lambda}}

\newcommand{\Hh}{\widehat{\mathbf{H}}}
\newcommand{\Qh}{\widehat{\mathbf{Q}}}
\newcommand{\Kh}{\widehat{\mathbf{K}}}
\newcommand{\Vh}{\widehat{\mathbf{V}}}
\newcommand{\qh}{\hat{q}}
\newcommand{\kh}{\hat{k}}
\newcommand{\vh}{\hat{v}}

\newcommand{\Hb}{\Bar{\mathbf{H}}}

\newcommand{\Vb}{\Bar{\mathbf{V}}}
\newcommand{\Xb}{\Bar{\mathbf{X}}}
\newcommand{\qb}{\Bar{q}}
\newcommand{\kb}{\Bar{k}}
\newcommand{\vb}{\Bar{v}}
\newcommand{\hb}{\Bar{h}}

\newcommand{\Wh}{\widehat{\mathbf{W}}}
\newcommand{\llo}{^{(\ell)}}
\newcommand{\llh}{^{(\ell+1/2)}}
\newcommand{\llt}{^{(\ell+3/4)}}
\newcommand{\lln}{^{(\ell+1)}}
\newcommand{\ah}{\widehat{a}}
\newcommand{\eps}{\varepsilon}
\newcommand{\norm}[1]{\lVert #1 \rVert}
\newcommand{\abs}[1]{\lvert #1 \rvert}
\newcommand{\tp}{^\mathsf{T}}

\newcommand{\w}[2]{\W_{#1}^{(#2)}}
\newcommand{\wh}[2]{\widehat{\W}_{#1}^{(#2)}}

\DeclareMathOperator{\bigO}{\mathcal{O}}
\DeclareMathOperator{\ReLU}{ReLU}

\DeclareMathOperator{\rank}{rank}

\title{A Theory for Compressibility of Graph Transformers for Transductive Learning}

\author{%
  Hamed Shirzad \\
  University of British Columbia\\
  \texttt{shirzad@cs.ubc.ca} \\
  \And
  Honghao Lin \\
  Carnegie Mellon University\\
  \texttt{honghaol@andrew.cmu.edu} \\
  \And
  Ameya Velingker \\
  Independent Researcher\footnotemark[1] \\
  \texttt{ameyav@gmail.com}
    \And
  Balaji Venkatachalam \\
  Meta\thanks{Work done in part while at Google.} \\
  \texttt{bave@meta.com}
  \And
  David P. Woodruff \\
  CMU \& Google Research \\
  \texttt{dwoodruf@cs.cmu.edu}
  \And
  Danica J.\ Sutherland \\
  UBC \& Amii\\
  \texttt{dsuth@cs.ubc.ca}
}

\begin{document}

\maketitle

\begin{abstract}
Transductive tasks on graphs differ fundamentally from typical supervised machine learning tasks, as the independent and identically distributed (i.i.d.) assumption does not hold among samples. Instead, all train/test/validation samples are present during training, making them more akin to a semi-supervised task. These differences make the analysis of the models substantially different from other models. Recently, Graph Transformers have significantly improved results on these datasets by overcoming long-range dependency problems. However, the quadratic complexity of full Transformers has driven the community to explore more efficient variants, such as those with sparser attention patterns. While the attention matrix has been extensively discussed, the hidden dimension or width of the network has received less attention. In this work, we establish some theoretical bounds on how and under what conditions the hidden dimension of these networks can be compressed. Our results apply to both sparse and dense variants of Graph Transformers.
\end{abstract}

\section{Introduction}

Graphs are a versatile data structure that naturally models relations between entities across many domains, such as social networks, citation graphs, and systems and code analysis. One common goal is node-level prediction on a single large graph, i.e., \emph{transductive node classification}. In this scenario, one is typically presented with a semi-supervised task in which some node labels are provided (for training or validation) and others are to be predicted. This setting is challenging, as the different nodes do not obey typical distributional assumptions (e.g., they are not independent), and the classification of one node can affect that of other nodes via the neighborhood structure. Examples of tasks in these settings include identifying malicious users in a social network, predicting protein functionality in a protein-protein interaction network, or categorizing products based on a co-purchase network \citep{ogbPaper, platonov2023critical,shchur2018pitfalls,jure2014snap}.

Many variants of Graph Neural Networks (GNNs) are able to tackle transductive tasks on graphs \citep{kipf2016semi,velickovic2018graph,hamilton2017inductive}. A recent leap in progress on GNNs has been the use of Transformers to perform message-passing operations in these networks. Transformers~\citep{VaswaniSPUJGKP17} have brought about revolutionary changes in several domains of machine learning, ranging from natural language processing \citep{VaswaniSPUJGKP17, devlin2018bert, zaheer2020big} to computer vision \citep{dosovitskiy2020image} and, more recently, geometric deep learning \citep{DwivediBresson21, kreuzer2021rethinking, Ying2021DoTR, rampavsek2022recipe, shirzad2023exphormer, shirzad2024spexphormer, muller2023attending}. The use of Transformers on graph data involves attention-based message passing on a certain computational graph; the computational graph can either be a ``full'' graph (where each node attends to every other node) or a sparse graph. Although the typical choice is a full graph, the resulting quadratic complexity is generally infeasible for sufficiently large graphs. Thus, for applications in which graphs are extremely large, one often uses Transformers on a sparse computation graph. Indeed, many linear sparse or low-rank approximations have been proposed for Transformers on graphs \citep{shirzad2023exphormer, shirzad2024spexphormer, wu2022nodeformer, deng2024polynormer}.

A significant body of research on scaling Transformers focuses on how a Transformer model can be sparsified or how low-rank approximations of the attention matrix can be calculated. However, an often overlooked aspect in these calculations is determining how large the hidden dimension needs to be, or, alternatively, how much this hidden dimension can be reduced. A Transformer with a hidden dimension $D$ operating on a graph with $n$ nodes and with an attention pattern consisting of $m$ attention edges (e.g., $m = \Theta(n^2)$ in a full Transformer), has a computational complexity of $\bigO(nD^2 + mD)$. In most previous works, $D$ is typically considered a constant and therefore dominated by $m$ and $n$. However, there are a number of theoretical works on GNNs show that, in order to achieve certain properties (e.g., expressivity), the hidden dimension often needs to depend on the graph size (i.e., $D$ is superconstant)~\citep{sanford2024representational,sanford2024transformers,sanford2024understanding,loukas2019graph}. One notable exception is the work of \citet{shirzad2024spexphormer}, in which a low-width network is first used to estimate attention scores, after which the estimates are used to sparsify the network and train a larger model with the computed sparsity pattern; that work, however, is primarily empirical, with only limited theoretical results (described in more detail below).

In this work, we address the aforementioned shortcomings and theoretically analyze the compressibility of a single-head Transformer model. Assuming a Graph Transformer is trained with some hidden dimension $D$, we seek to determine some bounds showing how small a compressed network can be while approximately preserving the outputs of the original network. In particular, we provide a series of results showing that, in various situations, the hidden dimension can be compressed substantially while preserving the output of the original network up to an additive error of $\bigO(\epsilon)$ and also preserving attention scores of the original network up to a $1\pm  \bigO(\epsilon)$ multiplicative approximation factor, for arbitrarily small $\epsilon$. We note that almost all of our results apply irrespective of the attention pattern, which can range from the sparsity pattern of the underlying graph structure of the data to a dense pattern, or even something in between, such as the patterns used by \citet{shirzad2023exphormer}. Furthermore, we complement our theoretical results with empirical results showing the existence of small networks with results that are competitive with those of a given large network.

\section{Notation \& Assumptions}
\label{sec:notations}

\paragraph{Graphs} Graphs are data structures consisting of a set of nodes $V$, and a set of edges $E$. Each edge in $E$ is an ordered pair of nodes. Undirected edges can be represented as two directed edges. Nodes typically have associated features represented by a matrix $\XX \in \R^{d_{in} \times n}$, where $d_{in}$ is the dimension of the input features for each node. 

We decouple the attention pattern from the graph structure and use $m$ to denote the number of attention edges. In a full Transformer, $m = n^2$. If we apply attention only over the graph edges, giving a model very similar to message-passing-based networks, $m$ would be the number of graph edges. We use $\mathcal{N}(v)$ to denote the set of neighbors of node $v$ in the attention graph.

\paragraph{Transformer Formulation} With some slight simplifications, we can formulate the $\ell$th layer as:
\begingroup\allowdisplaybreaks
\begin{align*}
    & \V\llo = \W_V\llo h\llo \qquad \Q\llo = \W_Q\llo h\llo \qquad \K\llo = \W_K\llo h\llo \\
    &a_{ij}^{(l)} = \frac{\exp\left({\mathbf{K}_j^{(\ell)} \cdot \mathbf{Q}_i^{(\ell)}}\right)}{\sum_{u \in \mathcal{N}_H(i)} \exp\left({\mathbf{K}_u^{(\ell)} \cdot \mathbf{Q}_i^{(\ell)}}\right)},\\
    &h_i\llh = \sum_{j \in \mathcal{N}(i)} a_{ij}^{(l)}\mathbf{V}_j^{(\ell)},\\
    &h_i\llt = \sigma \left(\W_1^{(\ell)} \left(h_i\llh\right)\right), \\
    &h_i\lln = \W_2^{(\ell)} h_i\llt.
\end{align*}
\endgroup
Here, $h_i\llo$ is the output of the layer $\ell-1$ for node $i$.
$\W_V\llo, \W_Q\llo, \W_K\llo, \W_1, \W_2 \in \R^{D \times D}$, except for the first layer where $\W_V^{(0)}, \W_Q^{(0)}, \W_K^{(0)} \in \R^{D \times d_{in}}$.
$\sigma$ can be any 1-Lipchitz elementwise activation function, such as $\ReLU$.
The standard self-attention formulation uses a
$\frac{1}{\sqrt{D}}$ normalization, but we ignore this since it can be combined with $\W_Q$ and $\W_K$ matrices.
To simplify the notation, we define $\Ho^{(0)} := \XX$. We use $\mathcal{T}(X)$ to refer to the large network trained with hidden dimension $D$ and has $L$ layers in total. $d$ always refers to a compressed network dimension size.

\paragraph{Other Notation} We use $\norm{x}$ to mean the Euclidean norm of a vector $x$ unless stated otherwise. We use capital bold letters for matrices and small letters with indices to refer to the columns of these matrices, e.g., $h_i\llo$ is the column $i$ of matrix $\Ho\llo$.
\cref{tab:notations} gives more.

\paragraph{Assumptions} We remove the normalization parts from the architecture but assume that in all input of the Transformer layers, $\norm{x_i}_2, \norm{h_i\llo}_2 \leq \sqrt{\alpha}$, and the operator norms of the matrices $\W_\cdot$ of each linear mapping are bounded by a constant $\beta$. In real Transformers, there is a \texttt{LayerNorm} between layers. 
The default node feature normalization function in both PyTorch-Geometric and DGL libraries, two of the most standard libraries for graph processing in PyTorch, normalize each node feature to $\norm{x_i}_1 = 1$ \citep{Fey/Lenssen/2019, wang2019deep}. Regarding the operator norm of the matrices, we also know that their operator norm is around two under standard initialization. If their operator norm does not change significantly during training, this assumption remains valid.

 Graph Transformers are typically fairly shallow; for example, Exphormer and Spexphormer usually employ networks with two to four layers \citep{shirzad2023exphormer,shirzad2024spexphormer}, while Nodeformer, Difformer, and SGFormer use at most three layers \citep{wu2022nodeformer,wu2023difformer,wu2024simplifying}.
 We thus assume $L = \bigO(1)$.

\section{Reducing the Attention Calculation Complexity Using JLT}

The attention calculation in practice is the most computationally expensive part of the model, because usually $m = \omega(n)$, and in full Transformers $m = n^2$. Also, $n$ is often much larger than the usual context length in natural language processing tasks, making the attention calculation part for Graph Transformers very costly. The attention calculation is the only part that depends on $m$; other parts only scale with $n$. In this section, we will show that $\W_Q, \W_K$ can be compressed into an $\R^{D \times d}$ or $\R^{d_{in} \times d}$ matrix for $d = \bigO(\frac{\log n}{\eps^2})$ with a bound of $\bigO(\eps)$ error on the output. This dimension reduction effectively reduces the computational complexity of the attention calculation part from $\bigO(mD)$ to $\bigO(md)$.
A version of this result also appeared in our recent work \citep{shirzad2024spexphormer}.

The Johnson-Lindenstrauss Lemma \citep{johnson1984extensions} is a powerful tool in theoretical computer science that helps preserve the pairwise distance between high-dimensional encodings when mapping them to a lower dimension. Under certain conditions, this pairwise distance preservation can also be applied to preserve the pairwise dot product. This dot product is present in the attention mechanism, allowing us to apply it there to compress the network.

\begin{lemma}[Johnson-Lindenstrauss Transform Lemma, \textit{JLT}]
Assume $0 < \epsilon, \delta < \frac{1}{2}$ and any positive integer $D$, if $d = \mathcal{O}(\frac{\log(1/\delta)}{\epsilon^2})$, there exists a distribution over matrices $\mathbf{M} \in \mathbb{R}^{d \times D}$ that for any $x \in \mathbb{R}^{D}$ and $\norm{x} = 1$,
\[
\operatorname{Pr}(\norm{\M x} - 1 > \epsilon) < \delta 
.\]
\end{lemma} 

This lemma immediately implies a dot product version: 

\begin{corollary} [JLT-dot product]
\label{cor:jltdot}
    Assume $0 < \epsilon, \delta < \frac{1}{2}$ and any positive integer $D$, if $d = \mathcal{O}(\frac{\log(1/\delta)}{\eps^2})$, there exists a distribution over matrices $\M \in \R^{d \times D}$ that for any $x, y \in \R^{D}$, and $\norm{x}, \norm{y} \leq \sqrt{\gamma}$,
\[
\operatorname{Pr}((1-\eps\gamma) x\tp y <  x\tp\M\tp\M y  < (1 + \eps\gamma) x\tp y) < \delta 
.\]
\end{corollary}

For a proof see e.g.\ \citet[Corollary 2.1]{Lec09}. As a result of this corollary, if we have $m$ pairs of vectors $(x_i, y_i)$, and for each pair $i$, $\norm{x_i}_2, \norm{y_i}_2 \leq \sqrt{\gamma}$, and $d = \mathcal{O}(\frac{\log(m)}{\eps^2})$, there exists an $\M$ such that for all these pairs $\abs{x_i\tp\M\tp\M y_i -  x_i\tp y_i}  <  \eps\gamma$. The proof can be done using a union bound over the error from Corollary~\ref{cor:jltdot}. Also, in our case where $m$ is the number of edges, we know that $m \leq n^2$, thus we can also say $d = \mathcal{O}(\frac{\log(n)}{\eps^2})$. Using this result, we prove the following theorem about Graph Transformers. For this theorem and all upcoming proofs, due to lack of space, we defer the proofs to the \cref{ap:proofs}. Please check \cref{proof:narrow_attention} for the proof of this theorem. In all proofs, we also prove that the attention scores from the compressed network are close to those of the reference network. This is mainly to show the consistency of the results with \cite{shirzad2024spexphormer} and the consistent explainability of the model through the attention scores.

\begin{theorem}
\label{thrm:narrow_attention}
There exists a Transformer $\widehat{\mathcal{T}}$, that for any layer $\W_Q$ and $\W_K$ are in $\R^{d \times D}$ for a $d=\mathcal{O}(\frac{\log n}{\eps^2})$, with a sufficiently small $\eps$, and for all $i \in [n]$, $\norm{\mathcal{T}(X)_i - \widehat{\mathcal{T}}(X)_i}_2 = \mathcal{O(\eps)}$. Furthermore, for any attention score, ${a_{ij}\llo}/\,{\ah_{ij}\llo} = 1 + \mathcal{O}(\eps)$.
\end{theorem}

 The asymptotic bounds on the Johnson-Lindenstrauss Lemma (JLT) are tight \citep{burr2018optimal}, and thus, without any extra assumptions, further compression is not feasible beyond some constant factors. However, in real-world graphs, the columns of $\XX$ are typically not $n$ distinct vectors, and many vectors may be equal or very similar to each other. If we have $\kappa$ unique vectors in the first layer, the complexity for $d$ can be reduced to $\mathcal{O}(\frac{\log \kappa}{\eps^2})$. Also, in many datasets, $d_{in}$ is quite small; we will see that if $\rank(\Ho\llo) \leq d$, we can reduce the hidden dimension in attention calculation to $d$. In later layers, for example in homophilic graphs, we can expect neighboring nodes to have very similar embeddings, and in these scenarios, we can expect more compression. Even in heterophilous graphs, if the neighborhood of nodes from the same class has a small diversity, we can expect very similar embeddings. We will investigate this further in the following sections.

\section{Exploring Low-rank Assumptions}

In many variants of GNNs, we observe that node embeddings rapidly lose rank and converge to a small number of possible embeddings. This convergence can occur in either a beneficial or detrimental manner. Many GNNs suffer from oversmoothing problems, where all node embeddings converge to a single embedding \citep{oono2019graph,nt2019revisiting}. However, there are also many beneficial scenarios where low-rank embeddings emerge, such as in graph coarsening/pooling methods and when embeddings for nodes from the same class converge to the same encodings. In many of these scenarios, an exact low-rank structure is rare, but embeddings that are \emph{nearly} low-rank are highly plausible.

These phenomena apply equally to the attention mechanism here, and so expecting approximately low-rank embeddings is reasonable. Additionally, in many datasets, the input matrix $\XX$ is actually of very low dimension: for example, the node features have size $\leq 10$ in the ogbn-proteins, Tolokers, and Minesweeper datasets \citep{ogbPaper,platonov2023critical}, and thus all $\Q, \K,$ and $\V$ matrices in the first layer have ranks no more than $10$.

\subsection{Low-rank Embeddings}

First, we show that if the input vectors are low-rank and the rank after each activation function is still small, we can at least compress to the maximum rank of the embeddings' mapping after the activation functions. The proof is in \cref{proof:narrow_network_low_rank_1}.

\begin{proposition}
\label{prop:narrow_network_low_rank_1}
Assume for inputs $X$ and for each $\ell$ $H\llt$, we have 
$\rank(X) \leq d$ and for all $\ell$, $\rank(\Ho\llt) \leq d$. 
There exists a Transformer $\widehat{\mathcal{T}}$, of width $d$, $\mathcal{T}(X)_i = \U_{out}\widehat{\mathcal{T}}(X)_i$, for some $\U \in \R^{D \times d}$. Furthermore, for any attention score $a_{ij}\llo = \ah_{ij}\llo$.
\end{proposition}

In this theorem, unlike the others, everything is perfectly reconstructible with no error propagation. However, this assumption is extremely strong; we will thus explore more realistic assumptions.

\subsection{Almost Low-rank Embeddings}
A more realistic assumption is that the embeddings are not exactly low-rank, but there is a low-rank matrix which is column-wise close. With this assumption, we can reduce all the dimensions, except for the activation function part, to dimensions of $d$. 

\begin{theorem}
\label{thrm:narrow_network_low_rank_2}
Assume for $X$ and $H\llt$ for each $\ell$, we have $\Bar{X}$
and $\Hb\llt$ such that for each $i \in [n]$, 
$\norm{x_i - \Bar{x}_i} \leq \eps$, 
$\norm{h\llt_i - \hb\llt_i} \leq \eps$, 
$\rank(\Xb) \leq d$ and for all $\ell$, $\rank(\Hb\llt) \leq d$. 
There exists a Transformer $\widehat{\mathcal{T}}$, which in each layer $\Wh_V, \Wh_Q$, and $\Wh_Q \in \R^{d \times d}$  and $W_1 \in \R^{d \times D}$ and $W_2 \in \R^{D \times d}$, with a sufficiently small $\eps$, and for all $i \in [n]$, $\norm{\mathcal{T}(X)_i - \widehat{\mathcal{T}}(X)_i}_2 = \mathcal{O}(\eps)$. Furthermore, for any attention score, ${a_{ij}\llo}/\,{\ah_{ij}\llo} = 1 + \mathcal{O}(\eps)$.
\end{theorem}

The proof for this theorem can be found in \cref{sec:proof_narrow_network_low_rank_2}. With this compression we do not have any $D\times D$ matrices, but there are a few size-$D$ embeddings appearing after the activation function. 
Similar approaches used in previous proofs will not help reduce this dimension. A brief reasoning and a counter-example explaining why this approach will not work are given in \cref{note:impossible_activation}.

If we relax the expectation of having \emph{all} nodes within the $\bigO(\eps)$ error, however, we can have \emph{most} nodes within the error bound by leverage score sampling. 
This gives at least 99\% chance for each node to be within the correct boundaries, but it cannot guarantee for \emph{all} nodes.

\begin{proposition}
\label{prop:leverage_score_sampling}
    Suppose that $\Ho \in \mathbb{R}^{D \times n}$ has a $\rank(d)$-approximation $\Ho'$ where for $i \in [n]$ we have $\norm{h_i - h'_i}_2 \le \eps$. Then there exists a row selection matrix $\mathbf{S} \in \mathbb{R}^{k \times D}$ with a matrix $\U \in \mathbb{R}^{D \times k}$ where $k = O(d \log d )$ such that for at least $0.99$-fraction of $i \in [n]$ we have 
    \[
    \norm{\U \mathbf{S} h_i - h_i}_2 \le \bigO(\eps) \; .
    \]
\end{proposition}

The proof for this proposition is in \cref{sec:proof_leverage_score_sampling}. Because the selection can be passed through the activation function, this can be combined with the activation function to reduce the dimension. However, if a node is not within the $\bigO(\eps)$ error, that means that all the nodes attending to it may not have similar attention as does the large network;
without further analysis, this then loses the guarantee for all nodes in the $L$-hop neighborhood of this node.

\subsection{Low-rank with Clustering Assumptions}
If instead of assuming the embeddings have a low-rank estimate, we assume they exactly can be clustered around maximum $d$ well-separated centers after each attention pooling operation, we can have a model with all linear mappings in $\R^{d \times d}$. A more formal theorem based on this idea is:

\begin{theorem}
\label{thrm:narrow_network_cluster}
    Assume the activation function in $\mathcal{T}$ is $\ReLU$, and in each layer after the attention operation we can cluster the vectors in $\Ho\llh$ into at most $d$ clusters with centers $c\llo_1, \cdots, c\llo_d$ with the condition that $0 < \gamma_1 < \norm{c\llo_a} < \gamma_2$, where $\gamma_1$ and $\gamma_2$ are constants, and for each $i$ there exist a $c\llo_a$ that $\norm{h_i\llh - c\llo_a} \leq \eps \norm{c\llo_a}$ for a sufficiently small $\eps$ and the clusters are well-separated in a way that for each two clusters $a \neq b$, $c\llo_a \cdot c\llo_b < \gamma_1^2/2$. If either 
    \begin{enumerate}
        \item $d \geq
        c\frac{\log n}{\eps^2}$, for a constant $c$ independent of the problem parameters,
        or
        
        \item there exists some approximation $\Bar{X}$ such that $\rank(\Bar{X}) \leq d$ and for each $i$, $\norm{X_i - \Bar{X}_i} \leq \eps$,
    \end{enumerate}
then there exists a Transformer $\widehat{\mathcal{T}}$, of width $d$, with a sufficiently small $\eps$, and for all $i \in [n]$, $\norm{\mathcal{T}(X)_i - \widehat{\mathcal{T}}(X)_i}_2 = \mathcal{O}(\eps)$. Furthermore, for any attention score, ${a_{ij}\llo}/\,{\ah_{ij}\llo} = 1 + \mathcal{O}(\eps).$
\end{theorem}

The proof for this theorem can be found in \cref{sec:proof_narrow_network_cluster}. The main idea of the proof is based on separating the nodes from different clusters using $d$ linear mappings in the activation function layers.

\section{Experiments}

\begin{figure*}[t!]
\vspace{-0.2in}
\subfloat[][]{\includegraphics[width = 1.88in]{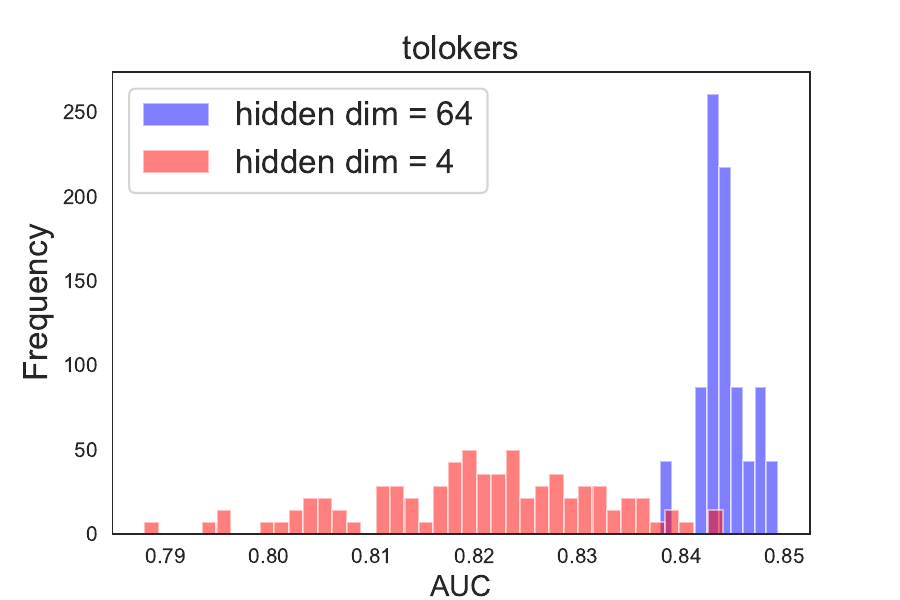}} 
\subfloat[][]{\includegraphics[width = 1.88in]{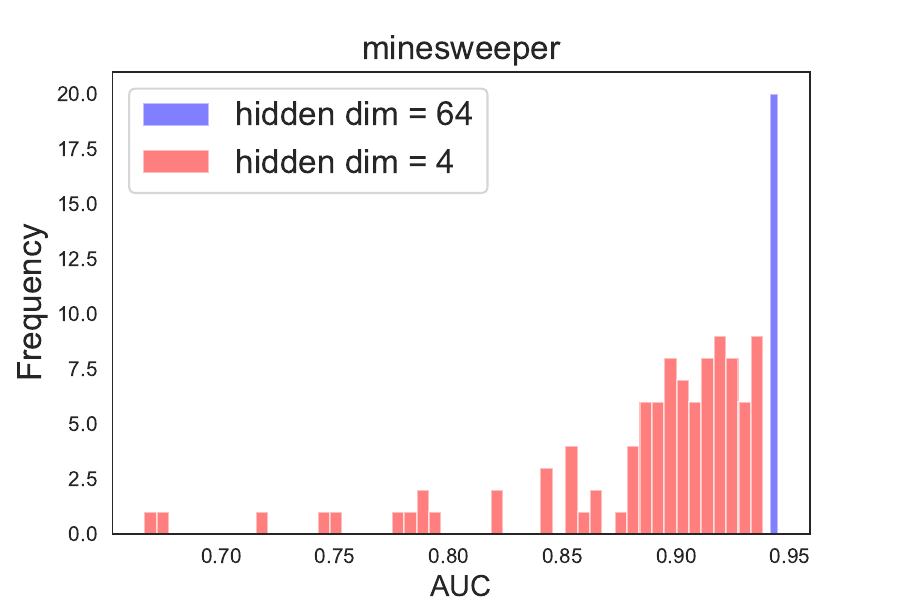}}
\subfloat[][]{\includegraphics[width = 1.88in]{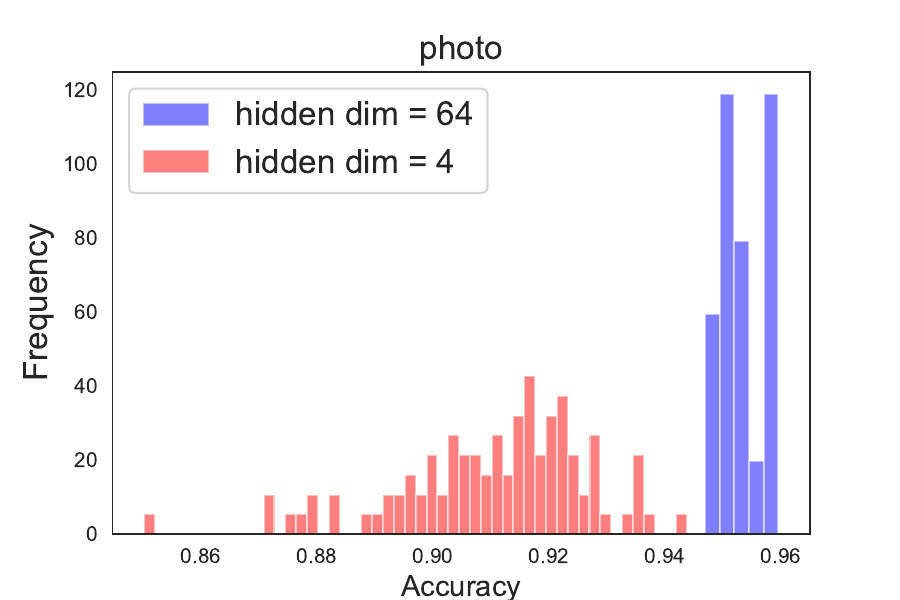}}
\caption{Comparison of the results from a relatively large network with hidden dimension 64 and a small network with hidden dimension 4.}
\label{fig:exp_results}
\vspace{-0.1in}
\end{figure*}

\begin{table}[t]
\caption{Average operator norm of the linear mappings in the network and average norm of the input vectors of the Transformer layers from reference large networks. All numbers are average $\pm$ std.}
    \centering
    \scalebox{0.8}{
    \begin{tabular}{l|ccc}
    \toprule
        Dataset & Tolokers & Minesweeper & Photo\\
        \midrule
        Operator Norm Average & $2.83 \pm 0.13$ & $2.62 \pm 0.07$ & $2.15 \pm 0.06$\\
        Vector Norm Average & $3.71 \pm 0.10$ & $3.44 \pm 0.21$ & $3.64 \pm 0.23$ \\
        \bottomrule
    \end{tabular}
    }
    \label{tab:norms}
\end{table}

In this section, we experimentally investigate whether models with significantly smaller dimensions can achieve high-quality results, and whether any small hidden dimension network can perform competitively with larger networks. While much of the theory in this work is constructive, some steps require extensive trial and error. To explore this question, we conduct the following experiment instead of directly compressing the actual network: we train a very small network with a hidden dimension of 4, using 100 different initializations. The Tolokers and Minesweeper datasets have very low homophily scores, but their input dimensions are also very low: ten for Tolokers and seven for Minesweeper. On the other hand, the Photo dataset has input node features of size 745 but exhibits a very high homophily score.

In \cref{fig:exp_results}, we compare the distribution of results from a relatively large network with a hidden dimension of 64 to that of the small network. As expected, gradient descent performs worse on average for the smaller network. However, on each dataset, the maximum test AUC or accuracy closely approaches the results of the large network, indicating the existence of a compressed network with competitive performance. Nevertheless, exploring initializations is impractical for large real-world datasets, emphasizing the need to learn to compress large models and leaving the investigation of practical compression algorithms as future work.

Furthermore, we measure the operator norm and vector norms in the large trained model to validate some assumptions in \cref{sec:notations}. The results, provided in \cref{tab:norms}, confirm that these norms are reasonably small. More details about the datasets, experimental setup, hyperparameters, and baseline comparisons can be found in \cref{sec:exp_details}.

\section{Conclusion, Limitations \& Future Work}
In this work, we have analyzed the compressibility of Graph Transformers under several assumptions. We have showed that under mild assumptions, the hidden dimension of the underlying attention calculation can be reduced to a level that is logarithmic in the number of nodes, and the low-rank approximation of matrices can also lead to compression of the model with minimal losses.

Although we have proved the existence of a compressed network in many scenarios, this does not imply that training with a gradient-based algorithm will necessarily lead to the introduced weights, but this gives at least the guarantee that such a network exists. However, if we can train the large model, many proofs in this work have been constructive, and even for existence proofs such as JLT-lemma, random creation of matrices would have a fair chance to be a valid result, and thus trial and error is possible with these works. However, we keep the experiments on these setups and develop efficient algorithms for compressing based on these theorems for future work.

The results in this work could potentially extend to other architectures such as low-rank attention methods and other GNN variants, opening avenues for future work. Additionally, while this work has primarily focused on scenarios where network compression is possible, impossibility results for compression have not been thoroughly explored. Finding tighter bounds on the hidden dimension of the compressed network would be an interesting problem for future work.

\begin{ack}
This work was supported in part by the Natural Sciences and Engineering Resource Council of
Canada,
the Fonds de Recherche du Québec - Nature et technologies (under grant ALLRP-57708-2022),
the Canada CIFAR AI Chairs program,
the BC DRI Group, Calcul Québec, Compute Ontario, and the Digital
Resource Alliance of Canada.
Honghao Lin was supported in part by a Simons Investigator Award, NSF CCF-2335412, and a CMU Paul and James Wang Sercomm Presidential Graduate Fellowship.
\end{ack}

\bibliography{bibliography}
\bibliographystyle{apalike}

\newpage

\appendix
\onecolumn

\section{Notation Table}

\begin{table}[h]
\caption{A summary of the notations used in this paper. The hat notation always refers to a compressed network equivalent of a vector or matrix from the reference network.}
\label{tab:notations}
\centering
\scalebox{0.9}{
\begin{tabular}{l|l} 
\toprule
Notation & Definition  \\ 
\hline
    $n$ & The number of nodes in the graph \\
    $m$ & The number of attention edges in total, including graph and expander edges \\
    $d$     &     Hidden dimension of a narrow network     \\
    $D$     &     Hidden dimension of the original large graph  \\
    $L$     &     The total number of layers  in the network    \\
    $\ell$     &     Arbitrary layer index     \\
    $\mathbf{V}$   &     Value mapping of the vectors in the attention mechanism     \\
    $\mathbf{Q}$   &     Query mapping of the vectors in the attention mechanism \\
    $\mathbf{K}$   &     Key mapping of the vectors in the attention mechanism \\
    $\mathbf{W}_\cdot^{(\ell)}$     &     Weight matrix of mapping such as key, query, value, edge features, or bias in layer $\ell$     \\
    $\widehat{\mathbf{W}}_\cdot^{(\ell)}$     &  Low dimensional network's weight matrix for a mapping in layer $\ell$     \\
    $\M_\cdot$ & A linear mapping matrix (usually from the higher dimension to the smaller)\\
    $\ReLU$     &    Rectified Linear Unit    \\
    $\Ho\llo$ & Output of layer $\ell-1$ from the reference network \\
    $\Hb\llo$ & A low-rank estimation of $\Ho\llo$ \\
    $\Hh\llo$ & Output of layer $\ell-1$ from a compressed network \\
    $h_i\llo$ & column $i$ of matrix $\Ho\llo$ \\
    $a_{ij}^{(\ell)}$ & The Attention score between nodes $i$ and $j$ in layer $\ell$\\
    $\hat{a}_{ij}^{(\ell)}$ & The attention score between nodes $i$ and $j$ in layer $\ell$ from a smaller network\\
\bottomrule
\end{tabular}
}
\end{table}

\section{Related Work}
\paragraph{GNNs \& Graph Transformers}  
Many variants of GNNs have been used for tackling transductive tasks on graphs. Examples include GCN \citep{kipf2016semi}, GAT \citep{velickovic2018graph}, and GraphSAGE \citep{hamilton2017inductive}. Recently, Graph Transformers have been developed to address the long-range dependencies that message-passing-based methods struggle to capture. Since full-Transformers are very costly, especially for graphs with millions of nodes, the focus has been on sparser patterns of the Transformer \citep{shirzad2023exphormer, shirzad2024spexphormer} or low-rank estimations of the attention matrix \citep{wu2022nodeformer, wu2023difformer, wu2019simplifying, deng2024polynormer}.

Oversmoothing is a problem rising in many variants of GNNs~\citep{oono2019graph,nt2019revisiting,roth2024preventing}. In this problem, all the embeddings converge to a single vector, or embeddings matrix loses its rank rapidly. The rank loss problem is not unique to message-passing --- attention mechanisms can suffer from a similar problem as well \citep{dong2021attention}. Lower rank embeddings are not always bad, many coarsening/pooling-based methods on graphs rely on combining nodes with similar features into single nodes to reduce the graph size and process only a much smaller graph \citep{liu2022graph}.

\paragraph{Theoretical Works}
Many recent works have theoretically analyzed the required hidden dimension or depth of the Transformers or GNNs for graph-related algorithms \citep{loukas2019graph,sanford2024representational, sanford2024transformers, sanford2024understanding}. Also, Johnson-Lindestrauss lemma have helped to analyze Transformers for investigating their power in learning diverse attention patterns \citep{likhosherstov2021expressive}. A variant of linear Transformer have been designed based on this lemma \citep{wang2020linformer}.

\section{Proofs}
\label{ap:proofs}

\subsection{Proof of \cref{thrm:narrow_attention}}
\label{proof:narrow_attention}
\begin{proof}
In the proof we use hat notation, $\widehat{\square}$, for the vectors and matrices from $\widehat{\mathcal{T}}$, for example, $\hat{h}^{(\ell)}$ are the outputs of layer $\ell$, and $\Wh_\cdot$ are the weight matrices for this network. In all layers for both networks $\W_V, \W_1$, and $\W_2$, are of the same size, so we set $\Wh_V = \W_V$, $\Wh_1 = \W_1$, and $\Wh_2 = \W_2$.

For the proof, we want to find $\eps^{(0)}, \cdots, \eps^{(L)}$ in a way that for any $v$ in layer $\ell$, $\abs{h_v^{(\ell)} - \hat{h}_v^{(\ell)}} < \eps^{(\ell)}$. We will find these bounds inductively, starting from the first layer. We have $\eps^{(0)} = 0$, as both networks have the same input, and we want to bound $\eps^{(\ell+1)}$ based on $\eps^{(\ell)}$. 

We have $\Q\llo = \W_Q\llo \Ho\llo$, $\K\llo = \W_K\llo \Ho\llo$ and assume 
$\Bar{\Q}\llo = \W_Q\llo \Hh\llo$, $\Bar{\K}\llo = \W_K\llo \Hh\llo$. 
Because of the operator norm of matrices $\W_Q$ and $\W_K$, for each $i$ we have $\norm{q_i\llo - \Bar{q}_i\llo} \leq \eps\llo \beta$ and $\norm{k_i\llo - \Bar{k}_i\llo} \leq \eps\llo \beta$. Also, we have $\norm{q\llo_i}, \norm{k\llo_i} \leq \beta\sqrt{\alpha}$, thus $\norm{\Bar{q}_i\llo}, \norm{\Bar{k}_i\llo} \leq \beta(\eps\llo + \sqrt{\alpha})$. Now, for each pair of $i$ and $j$, we have:

\begin{align*}
    \abs{q_i\llo\cdot k_j\llo - \Bar{q}_i\llo \cdot \Bar{k}_j\llo} &= \abs{q_i\llo\cdot k_j\llo - \Bar{q}_i\llo \cdot k_j\llo + \Bar{q}_i\llo \cdot k_j\llo - \Bar{q}_i\llo \cdot \Bar{k}_j\llo} \\
    & \leq \abs{q_i\llo\cdot k_j\llo - \Bar{q}_i\llo \cdot k_j\llo} + \abs{\Bar{q}_i\llo \cdot k_j\llo - \Bar{q}_i\llo \cdot \Bar{k}_j\llo} \\
    & = \abs{(q_i\llo - \Bar{q}_i\llo) \cdot k_j\llo} + \abs{\Bar{q}_i\llo \cdot (k_j\llo - \Bar{k}_j\llo)} \\
   & \leq \norm{q_i\llo - \Bar{q}_i\llo} \norm{k_j\llo} + \norm{\Bar{q}_i\llo} \norm{k_j\llo - \Bar{k}_j\llo} \\
    & \leq \sqrt{\alpha}\beta\eps\llo + (\sqrt{\alpha}+\beta\eps\llo)\beta\eps\llo \\
    & = 2\sqrt{\alpha}\beta\eps\llo + (\beta\eps\llo)^2
\end{align*}

On the other hand, according to the ~\ref{cor:jltdot}, for a $0 < \eps < 1/2$ and $d=\mathcal{O}(\frac{\log(n)}{\eps^2})$ there exists a matrix $\M_{QK} \in \R^{d \times D}$, such that if we define $\Qh\llo = \M_{QK}\Bar{\Q}\llo$ and $\Kh\llo = \M_{QK}\Bar{\K}\llo$, $\abs{\Bar{q}_i\llo \cdot \Bar{k}_j\llo - \qh_i\llo \cdot \kh_j\llo} < \beta^2(\alpha+(\eps\llo)^2+2\sqrt{\alpha}\eps\llo)\eps$ for all $(i, j)$ pairs in the attention pattern. Note that we can define $\Wh_Q\llo = \M_{QK}\llo\W_Q\llo$, and $\Wh_K\llo = \M_{QK}\llo\W_K\llo$, both in $\R^{d \times D}$, as weights for the narrow attention score estimator network. With a triangle inequality we have $$\abs{q_i\llo \cdot k_i\llo - \qh_i\llo \cdot \kh_i\llo} < \beta^2(\alpha+(\eps\llo)^2+2\sqrt{\alpha}\eps\llo)\eps + 2\sqrt{\alpha}\beta\eps\llo + (\beta\eps\llo)^2.$$

By setting $\eps\llo \leq 1$, we have $$\abs{q_i\llo \cdot k_i\llo - \qh_i\llo \cdot \kh_i\llo} < \beta^2(\alpha+1+2\sqrt{\alpha})\eps + \beta(2\sqrt{\alpha}+\beta)\eps\llo.$$
Let us define $\eps_a = \beta^2(\alpha+1+2\sqrt{\alpha})\eps + \beta(2\sqrt{\alpha}+\beta)\eps\llo$, we have:

\begin{gather*}
\ah_{ij}^{(\ell)} = \frac{\exp(\qh_i\llo\cdot \kh_j\llo)}{\sum_{u \in \mathcal{N}_H(i)} \exp(\qh_i\llo \cdot\kh_u\llo)} \leq \frac{\exp(q_i\llo\cdot k_j\llo + \eps_a)}{\sum_{u \in \mathcal{N}_H(i)} \exp(q_i\llo\cdot k_j\llo -\eps_a)} \leq a_{ij}^{(\ell)}\exp(2\eps_a)
\\
\ah_{ij}^{(\ell)} = \frac{\exp(\qh_i\llo\cdot \kh_j\llo)}{\sum_{u \in \mathcal{N}_H(i)} \exp(\qh_i\llo\cdot \kh_u\llo)} \geq \frac{\exp (q_i^{(\ell)}\cdot k_j^{(\ell)} -\eps_a)}{\sum_{u \in \mathcal{N}_H(i)} \exp(q_i^{(\ell)} \cdot k_u^{(\ell)} +\eps_a)} \geq a_{ij}^{(\ell)}\exp(-2\eps_a)
\end{gather*}
Take note that if $2\eps_a < 1$, we have $\exp(2\eps_a) < 1 + 2\eps_a$ and $\exp(-2\eps_a) > 1 - \eps_a$, and for any $i, j$, $a\llo_{i,j} \leq 1$. Thus for any $i$ and $j$, $\frac{a\llo_{i,j}}{\ah\llo_{i,j}} = 1 + \bigO(\eps_a) = 1 + \bigO(\eps\llo)$.

Now we bound $\norm{h_i\llh - \hh_i\llh}$: 

\begin{align*}
   \norm{h_i\llh - \hh_i\llh} &= \norm{\sum_{j\in Nei(i)} a_{ij}\llo v_j\llo - \ah_{ij}\vh_j\llh} \\
    & = \norm{\sum_{j\in Nei(i)} a_{ij}\llo v_j\llo - \ah_{ij}\llo v_j\llo + \ah_{ij}\llo v_j\llo - \ah_{ij}\vh_j\llo} \\
    &= \norm{\sum_{j\in Nei(i)} (a_{ij}\llo - \ah_{ij}\llo) v_j\llo + \ah_{ij}\llo(v_j\llo - \vh_j\llo)} \\
    &= \norm{(v_j\llo - \vh_j\llo) + v_j\llo\sum_{j\in Nei(i)} (a_{ij}\llo - \ah_{ij}\llo)} \\
    & \leq \norm{v_j\llo - \vh_j\llo} + \norm{v_j\llo}\sum \abs{a_{ij}\llo - \ah_{ij}\llo} \\
    & \leq \eps\llo\beta + \sqrt{\alpha} \sum \max (1-\exp(-2\eps_a), \exp(2\eps_a) -1) a_{ij}\llo \\
    & \leq \eps\llo\beta + \sqrt{\alpha} (\exp(2\eps_a) -1),
\end{align*} 

and since $1+x < \exp(x) < 1+2x$ for $0<x<1$, if we have $\eps_a < 1$, we have 
\begin{equation}
    \norm{h_i\llh - \hh_i\llh} \leq \beta\eps\llo + 4 \sqrt{\alpha} \eps_a
\end{equation}

For the feed-forward network part, we know that this network is $\beta^2$-Lipschitz because $\W_1\llo$ and $\W_2\llo$ have maximum operator norm $\beta$ and $\sigma$ is a 1-Lipschitz activation function. Thus we have
\begin{equation*}
    \norm{h_i\lln - \hh_i\lln} \leq \beta^2(\beta\eps\llo + 4 \sqrt{\alpha} \eps_a) =
    (\beta^3 + 8\beta\alpha+4\beta^2\sqrt{\alpha})\eps\llo + 4\beta^2(\alpha\sqrt{\alpha} + 2\alpha + \sqrt{\alpha}) \eps.
\end{equation*}

Both $\beta^3 + 8\beta\alpha+4\beta^2\sqrt{\alpha}$ and $4\beta^2(\alpha\sqrt{\alpha} + 2\alpha + \sqrt{\alpha})$ are constants, and if we define them as $c_1$ and $c_2$, we have 

\begin{equation*}
    \eps\lln \leq c_1\eps\llo + c_2\eps
\end{equation*}

Given $\eps^{(0)} = 0$, as both networks get the same input, we have
\begin{align*}
    \eps^{(L)} &\leq c_1\eps^{(L-1)} + c_2\eps \\
    & \leq c_1(c_1\eps^{(L-2)} + c_2\eps) + c_2\eps \\
    & \cdots \\
    & \leq c_2\eps (c_1^{L-1} + \cdots + c_1) \\
    & = \frac{c_1(c_2^L-1)}{c_2-1}\eps
\end{align*}

While the error increases exponentially with the number of layers, when we have $L = O(1)$, then the error is bounded by a constant factor of chosen $\eps$. Now, we know that $\norm{\mathcal{T}(X)_i - \widehat{\mathcal{T}}(X)_i}_2 \leq \eps^{(L)} = \mathcal{O(\eps)}$.

This also holds that $\eps\llo = \bigO(\eps)$ for each $\ell$, thus $\frac{a_{ij}\llo}{\ah_{ij}\llo} = 1 + \mathcal{O}(\eps).$ \end{proof}

\subsection{Proof of \cref{prop:narrow_network_low_rank_1}}
\label{proof:narrow_network_low_rank_1}
\begin{proof}
    First of all, take note that any matrix $\mathbf{B}\in \R^{D \times n}$ of rank $d$, we have $\U \in \R^{D \times d}$ and $\Lam \in \R^{d \times D}$ that $\U\Lam \mathbf{B} = \mathbf{B}$. Particularly, $\Lam$ can be a selection of rows from $\mathbf{B}$ covering the whole span of columns of $\mathbf{B}$, and $\U$ will be a linear combination of the selected rows making columns of $\mathbf{B}$.
    
    We will make embeddings in each layer in a way that the embeddings from the low-dimensional network can be mapped linearly to the high dimension $D$ to recreate the embeddings from the high-dimensional network.

    Let us assume for each layer $\ell$ the input of the small network is $\Hh\llo \in \R^{d \times n}$, such that for a $\U \in \R^{D \times d}$, $\U_{in}\llo \Hh\llo = \Ho\llo$. Now, we will create the weights for the layer $\ell$ of $\widehat{\mathcal{T}}$ to have sizes $d\times d$, and the output of the layer can be mapped with a linear map $\U\llo_{out}$ to the outputs of the layer $\ell$ from $\mathcal{T}$.

    First, we will show the possibility of consistency in the attention scores. The attention scores in the matrix are a sparse version of $\mathbf{A} = \Ho\tp\W_Q\tp\W_K\Ho$, for this part, since all the matrices and representations are in layer $\llo$, we remove $\llo$ superscripts for the brevity of the writing. We need a $\Wh_Q, \Wh_K \in \R^{d \times d}$ that $\widehat{\mathbf{A}} = \Hh\tp\Wh_Q\tp\Wh_K\Hh$ gives us a similar attention matrix. We have:
    $$\mathbf{A} = \Ho\tp\W_Q\tp\W_K\Ho = \Hh\tp\U\tp\W_Q\tp\W_K\U\Hh.$$ Take note that $\U\tp\W_Q\tp\W_K\U$ is a matrix of shape $d \times d$. Now, if we have $\Wh_K = \U\tp\W_Q\tp\W_K\U$ and $\Wh_Q = I_d$, where $I_d$ is the identity matrix of size $d$, we have $\widehat{\mathbf{A}} = \Hh\tp\Wh_Q\tp\Wh_K\Hh = \mathbf{A}$.

Now for compressing the $\W_V$, we have $\V = \W_V\Ho = \W_V\U\Hh$. Since $\rank(\Ho) \leq d$, $\rank(\V) \leq d$. Thus there should be $\U_V \in \R^{D \times d}$ and $\Lam_V \in \R^{d \times D}$ that $\U_V \Lam_V \V = \V$. We have $\V = \U_V \Lam_V \W_V\U\Hh$ and thus we can take $\Wh_V = \Lam_V \W_V \U$, and thus $\Wh_V \in \R^{d \times d}$. Also, $\Vh = \Wh_V\Hh$, and we have $\V = \U_V\Vh$.

We have $\Ho\llh = \V\mathbf{A}$, and we will have $\Hh = \Vh\mathbf{A} = \U_V \V \mathbf{A}$. Thus, $\Ho\llh = \U_V\Hh\llh$. 

The next linear mapping comes with an activation function; thus, we can not do exactly the same trick as $\V$ mapping. We have $\Ho\llt = \sigma(\W_1\Ho\llh) = \sigma(\W_1\U_V\Hh\llh)$. Now, since we know $\rank(\Ho\llt) \leq d$, we have $\U_\sigma \in \R^{D \times d}$ and $\Lam_\sigma \in \R^{d \times D}$ in a way that $\U_\sigma\Lam_\sigma\Ho\llt = \Ho\llt$. In this case, we construct $\U_\sigma$ and $\Lam_\sigma$ so that we can also reduce the size of $\W\llo_1$. To construct this, we choose $\Lam_\sigma$ to have each row as a 1-hot vector, selecting maximum $d$ rows from $\Ho\llt$ that any other rows in the $\Ho\llt$ can be constructed from a linear combination of these rows. In this construction, we have $\Lam_\sigma \sigma(\W_1h\llh) = \sigma(\Lam_\sigma\W_1h\llh)$, since $\Lam_\sigma$ just selects rows from the mapped value and the activation function is an element-wise function. Thus we can have $\Wh\llo_1 = \Lam_\sigma\W\llo_1 \U_V \in \R^{d \times d}$. And we have $\Ho\llt = \U_\sigma \Hh\llt$.

The next linear layer exactly similar to the $V$ mapping can be compressed in a way that we have $\Ho\lln = \U_{out}\Hh\lln$. Now by induction since the assumption is correct for the input of the network, and if the assumption holds for the input we can have a compressed layer that the assumption will hold for the output layer, and the output of each layer is the input of the next layer the theorem is proved.
\end{proof}

\subsection{Proof of \cref{thrm:narrow_network_low_rank_2}}
\label{sec:proof_narrow_network_low_rank_2}
We will first prove the following lemma which will be used in several upcoming theorems:

\begin{lemma}
    \label{lem:low_rank_est_map}
        If $\Ho, \Hb \in \R^{D \times n}$, $\rank(\Hb) = d$, and for each row $i$, $\norm{\Ho_i - \Hb_i} \leq \eps$ for some value $\eps$, there exist matrices $\U \in \R^{D \times d}$ and $\Lam \in \R^{d \times D}$ such that $\forall i: \norm{\Ho_i - \U\Lam\Ho_i} \leq \eps$.
    \end{lemma}

    \begin{proof}
        Take $\Lam$ to be $d$ base vectors of size $D$ making column span of $\Hb$. Each column of $\Ho$ has maximum distance of $\eps$ from this span since $\norm{\Ho_i - \Hb_i} \leq \eps$. Thus we can have a $\U \in \R^{D \times d}$ that $\U\Lam\Ho$ will be the projection of the $\Ho$ to the span of columns of $\Hb$, and since this projection is the minimum distance we have $\norm{\U\Lam\Ho_i - \Ho_i} \leq \norm{\Hb_i - \Ho_i} \leq \eps$.
    \end{proof}

\begin{proof}[Proof of the theorem]

    We will prove this theorem inductively by making the embeddings in each layer in a way that the embeddings from the low-dimensional network can be mapped to the high dimension $D$ to approximate the high-dimensional network. We will prove the following lemma that helps both in the first layer and dimension reduction for $\W_2$ mappings.

    Thus, because input $\XX$ has a low-rank estimation $\Xb$, we can make $\widehat{\XX} = \Lam\XX$, in a way that there exist a matrix $\U$ that $\forall i: \norm{\U\hat{x}_i - x_i} \leq \eps$. For the convenience of the writing, we take $\Ho^{(0)}$ as $\XX$ and $\Hh^{(0)}$ as $\widehat{\XX}$. Now, we have a $\U$ that $\Hh^{(0)}$ that $\norm{\U\llo\hh^{(0)} - h^{(0)}_i} \leq \eps$. We will prove the step of the induction with the following lemma:

    For the step of our induction assume in a layer such as $\ell$, we have input $\Hh\llo \in \R^{d \times n}$ such that there exist $\U\llo$ that $\forall i: \norm{\U\llo\hh\llo - h\llo_i} \leq \eps\llo$, there exist a Transformer layer of width $d$ as the next layer that there exist $\U\lln$ that $\forall i: \norm{\U\lln\hh\lln_i - h\lln_i} \leq c_1\eps\llo + c_2\eps$ for constants $c_1, c_2 = \bigO(1)$.

    First, we will show the possibility of consistency in the attention scores. The attention scores in the matrix are a sparse version of $\mathbf{A} = \Ho\tp\W_Q\tp\W_K\Ho$, for this part, since all the matrices and representations are in layer $\llo$, we remove $\llo$ superscripts for the brevity of the writing. We need a $\Wh_Q, \Wh_K \in \R^{d \times d}$ that $\hat{\mathbf{A}} = \Hh\tp\Wh_Q\tp\Wh_K\Hh$ gives us a similar attention matrix. We start by estimating the $\mathbf{A} = \Ho\tp\W_Q\tp\W_K\Ho$ with $\Bar{\mathbf{A}} = \Hb\tp\W_Q\tp\W_K\Hb$. Now, we know that $\Hh = \Lam\Hb$ and $\U\Hh = \Hb$, thus $\Hb = \U\Lam\Hb$. By replacing this in the attention estimation we have 
    
    \begin{equation*}
        \Bar{\mathbf{A}} = \Hb\tp\Lam\tp\U\tp\W_Q\tp\W_K\U\Lam\Hb = \Hh\tp\U\tp\W_Q\tp\W_K\U\Hh.
    \end{equation*}

    Take note that $\U\tp\W_Q\tp\W_K\U$ is a $d \times d$ matrix. Now, if we define $\Wh_K = \U\tp\W_Q\tp\W_K\U$ and $\Wh_Q = \mathbf{I}_d$, where $\mathbf{I}_d$ is the identity matrix of size $d$, we have 
    $$\hat{\mathbf{A}} = \Hh\tp\Wh_Q\tp\Wh_K\Hh = \Hb\tp\W_Q\tp\W_K\Hb = \Bar{\mathbf{A}}.$$

    We have $\Q\llo = \W_Q\llo \Ho\llo$, $\K\llo = \W_K\llo \Ho\llo$ and if we define 
$\Bar{\Q}\llo = \W_Q\llo \Hb\llo$, $\Bar{\K}\llo = \W_K\llo \Hb\llo$. Very similar to what we saw in \ref{thrm:narrow_attention}, 
because of the operator norm of matrices $\W_Q$ and $\W_K$, for each $i$ we have $\norm{q_i\llo - \Bar{q}_i\llo} \leq \eps\llo \beta$ and $\norm{k_i\llo - \Bar{k}_i\llo} \leq \eps\llo \beta$. Also, we have $\norm{q\llo_i}, \norm{k\llo_i} \leq \beta\sqrt{\alpha}$, thus $\norm{\Bar{q}_i\llo}, \norm{\Bar{k}_i\llo} \leq \beta(\eps\llo + \sqrt{\alpha})$. Now, as we proved in \ref{thrm:narrow_attention} we have:

\begin{align*}
    \abs{q_i\llo\cdot k_j\llo - \Bar{q}_i\llo \cdot \Bar{k}_j\llo} \leq 2\sqrt{\alpha}\beta\eps\llo + (\beta\eps\llo)^2 \leq \beta(2\sqrt{\alpha}+\beta)\eps\llo.
\end{align*}

The last inequality is because we have $\eps\llo \leq 1$, which holds for a sufficiently small $\eps$, as we will see toward the end of the proof.

Now if we define $\eps_a = \beta(2\sqrt{\alpha}+\beta)\eps\llo$, we have:

\begin{gather*}
\ah_{ij} =\Bar{a}_{ij} = \frac{\exp(\qb_i \cdot \kb_j)}{\sum_{u \in \mathcal{N}_H(i)} \exp(\qb_i \cdot\kb_u)} \leq \frac{\exp(q_i\cdot k_j + \eps_a)}{\sum_{u \in \mathcal{N}_H(i)} \exp(q_i\cdot k_j -\eps_a)} \leq a_{ij}\exp(2\eps_a)
\\
\ah_{ij} = \Bar{a}_{ij} = \frac{\exp(\qb_i\cdot \kb_j)}{\sum_{u \in \mathcal{N}_H(i)} \exp(\qb_i\cdot \kb_u)} \geq \frac{\exp (q_i\cdot k_j -\eps_a)}{\sum_{u \in \mathcal{N}_H(i)} \exp(q_i \cdot k_u +\eps_a)} \geq a_{ij}\exp(-2\eps_a)
\end{gather*}

Take notice that if $2\eps_a < 1$, we have $\exp(2\eps_a) < 1 + 2\eps_a$ and $\exp(-2\eps_a) > 1 - \eps_a$, and for any $i, j$, $a\llo_{i,j} \leq 1$. Thus for any $i$ and $j$, $\frac{a\llo_{i,j}}{\ah\llo_{i,j}} = 1 + \bigO(\eps_a) = 1 + \bigO(\eps\llo)$. 

Now for compressing the $\W_V$, we start from $\Hb$ and we have $\Vb = \W_V\Hb = \W_V\U\Hh$. Now, according to the operator norm of $\W_V$, we know that $\max_i \norm{v_i - \vb_i} \leq \beta\eps\llo$. On the other hand since $\rank(\Hb) \leq d$, $\rank(\Vb) \leq d$. Thus there should be $\U_V \in \R^{D \times d}$ and $\Lam_V \in \R^{d \times D}$ that $\U_V \Lam_V \Vb = \Vb$. Let us define $\Vh = \Lam_V \Vb$. Then we have 
$$ \Vb = \W_V\U\Hh =  \U_V \Lam_V \W_V \U \Hh.$$

Now take $\Wh_V = \Lam_V \W_V \U$, and thus $\Wh_V \in \R^{d \times d}$. Also, $\Vh = \Wh_V\Hh$, and we have $\Vb = \U_V\Vh$, and thus $\max_i\norm{v_i - \U_V\vh_i} \leq \beta \eps\llo$.

Very similar to the proof in ~\cref{thrm:narrow_attention}, we can bound the $\norm{h_i\llh - \U_V\hh_i\llh}$: 

\begin{align*}
   \norm{h_i\llh - \U_V\hh_i\llh} &= \norm{\sum_{j\in Nei(i)} a_{ij}\llo v_j\llo - \ah_{ij}\U_V\vh_j\llh} \\
    & = \norm{\sum_{j\in Nei(i)} a_{ij}\llo v_j\llo - \ah_{ij}\llo v_j\llo + \ah_{ij}\llo v_j\llo - \ah_{ij}\U_V\vh_j\llo} \\
    &= \norm{(v_j\llo - \U_V\vh_j\llo) + v_j\llo\sum_{j\in Nei(i)} (a_{ij}\llo - \ah_{ij}\llo)} \\
    & \leq \norm{v_j\llo - \U_V\vh_j\llo} + \norm{v_j\llo}\sum \abs{a_{ij}\llo - \ah_{ij}\llo} \\
    & \leq \beta\eps\llo + \sqrt{\alpha} \sum \max (1-\exp(-2\eps_a), \exp(2\eps_a) -1) a_{ij}\llo \\
    & \leq \beta \eps\llo + \sqrt{\alpha} (\exp(2\eps_a) -1),
\end{align*} 

and since $1+x < \exp(x) < 1+2x$ for $0<x<1$, if we have $\eps_a < 1$, we have 
$$\norm{h_i\llh - \U_V\hh_i\llh} \leq \beta\eps\llo + 4 \sqrt{\alpha} \eps_a = \beta(1+ (8\alpha+\beta\sqrt{\alpha}))\eps\llo$$

For the convenience of writing we take $\eps_b = \beta(1+ (8\alpha+\beta\sqrt{\alpha}))\eps\llo$. For the feedforward network part, we know $\W_1$ has operator norm $\beta$ and $\sigma$ is 1-Lipschitz. Thus for each $i$ we have, $$\norm{\sigma(\W_1h_i\llh) - \sigma(\W_1\U_V\hh_i\llh)} \leq \beta \eps_b.$$
Now, we take $\Wh_1 = \W_1\U_V$ and this will give us $$\norm{\sigma(\W_1h_i\llh) - \sigma(\Wh_1\hh_i\llh)} \leq \beta \eps_b.$$ 
Also, we know that $\forall i: \norm{h\llt_i - \hb\llt_i} \leq \eps$, thus with the triangle inequality, we have $\forall i: \norm{\hh\llt_i - \hb\llt_i} \leq \eps + \beta\eps_b$. $\W\llo_2$ has an operator norm less than or equal to $\beta$, thus 
$$\forall i: \norm{\W\llo_2 \hb\llt_i - \W\llo_2 \hh\llt_i} \leq \beta\eps + \beta^2\eps_b.$$ 

Since $\rank(\Hb\llt) \leq d$, then $\rank(\W\llo_2\Hb) \leq d$. Thus $\W\llo_2\Hh\llt$ has a lower rank approximation with a column-wise maximum distance of $\beta\eps + \beta^2\eps_b$. According to the \cref{lem:low_rank_est_map}, we can have $\U\lln \in \R^{D \times d}$ and $\Lam\lln \in \R^{d \times D}$ that, $$\forall i: \norm{\W\llo_2\hh\llt_i - \U\lln\Lam\lln\W\llo_2\hh\llt_i} \leq \beta\eps + \beta^2\eps_b.$$ 
Now, take $\Wh\llo = \Lam\lln\W\llo_2$, and this will give us $\Hh\lln = \Wh\llo_2\Hh\llt \in \R^{d \times n}$ that 
$$\forall i: \norm{\U\lln\hh\lln_i - \W\llo_2\hh\llt_i} \leq \beta\eps + \beta^2\eps_b.$$ 

We also know that $\forall i: \norm{\W\llo_2 h\llt_i - \W\llo_2\hh\llt_i} \leq \beta^2\eps_b$, and $\Ho\lln = \W\llo_2 \Ho\llt$, thus:
$$\forall i: \norm{h\lln_i - \W\llo_2\hh\llt_i} \leq \beta^2\eps_b.$$

Combining the results with the triangle inequality we have:

$$\forall i: \norm{h\lln_i - \U\lln\hh\lln_i} \leq \beta\eps + 2\beta^2\eps_b = \beta\eps + \beta^3(1+ (8\alpha+\beta\sqrt{\alpha}))\eps\llo.$$

Thus, if we take $c_1 = \beta^3(1+ (8\alpha+\beta\sqrt{\alpha}))$ and $c_2 = \beta$, we have $c_1, c_2 = \bigO(1)$, and thus the lemma will be proven. This will prove the induction step.

Now, by induction since the assumptions are correct for the first layer. The assumptions also hold for the input and the assumptions being correct for a layer will result in it being correct for the following layer, for each layer $\eps\lln \leq c_1\eps\llo + c_2\eps$ and $\eps^{(0)} = \eps$. Since the number of layers is $L = \bigO(1)$, very similar to \cref{thrm:narrow_attention}, we have $\norm{\mathcal{T}(X)_i - \widehat{\mathcal{T}}(X)_i}_2 \leq \eps^{(L)} = \mathcal{O(\eps)}$ and $\frac{a_{ij}\llo}{\ah_{ij}\llo} = 1 + \mathcal{O}(\eps)$. \end{proof}

\subsubsection{Note on impossibility of the reduction using $U\Lam$ mapping} \label{note:impossible_activation} If we want to also decrease the dimension on the feed-forward layer, similar techniques we used to decrease the dimension of the linear mappings will fail due to the non-linearity of the activation function. 

Because $\rank(\Hb\llt) \leq d$, we can have $\U_\sigma \in \R^{D \times d}$ and $\Lam_\sigma \in \R^{d \times D}$ in a way that $\U_\sigma\Lam_\sigma\Hb\llt = \Hb\llt$. In this case, we construct $\U_\sigma$ and $\Lam_\sigma$ so that we can also reduce the size of $\W\llo_1$. To construct this, we choose $\Lam_\sigma$ to have each row as a 1-hot vector, selecting maximum $d$ rows from $\Hb\llt$ that any other rows in the $\Hb\llt$ can be constructed from a linear combination of these rows. In this construction, we have $$\Lam_\sigma \sigma(\W_1h\llh) = \sigma(\Lam_\sigma\W_1h\llh)$$, since $\Lam_\sigma$ just selects rows from the mapped value and the activation function is an element-wise function. Thus we can have $\Wh\llo_1 = \Lam_\sigma\W\llo_1 \U_V$. Furthermore, we have 
$$\norm{\Lam_\sigma\hb\llt_i - \hh\llt_i} \leq \eps + \beta^2\eps\llo + 4 \sqrt{\alpha}\beta \eps_a,$$
since the distance can not increase by selecting a subset of rows. Now, we have 
$$\norm{\U_\sigma (\Lam_\sigma\hb\llt_i - \hh\llt_i)} \leq \norm{\U_\sigma}_{op}(\eps + \beta^2\eps\llo + 4 \sqrt{\alpha}\beta \eps_a).$$ 

Now, if we can choose $\Lam_\sigma$ and $\U_\sigma$ in a way that $\U_\sigma$ has an $\bigO(1)$ operator norm this can give the contraction for the output of $\Wh_1\llo$ and input of $\Wh_2\llo$. Even a simpler condition that just preserves the distance with some constant around $\hb\llt_i$ vectors in a way that
$$\norm{\U_\sigma (\Lam_\sigma\hb\llt_i - \hh\llt_i)} \leq c\norm{\Lam_\sigma\hb\llt_i - \hh\llt_i}$$
for some constant $c$ will lead to a lower dimension feed-forward network. However, this is not always correct.

A counter-example to show that this is not always correct is that if $D = n$ and $\Hb \in \R^{D \times n}$ is a matrix with all elements equal to $1/\sqrt{D}$ and $\Ho = \Hb + I\eps$, then $\rank(\Hb) = 1$, but selecting just one row of $\Ho$ to estimate the whole matrix should use a constant, $c$ multiplication of that row for estimating all other rows. If this constant is one, then one column will have an error of $D\eps$, otherwise, the error will be at least $\max(\abs{1-c}, c+c\eps -1)D = \theta(\eps D)$. However, in the scenario that $c=1$, all columns except for one will have maximum distance $\bigO(\eps)$. However, it is also noteworthy that this counter-example is not a very likely thing to happen in the training of deep neural networks, at least in the presence of regularizers such as dropout and layer-norm we expect the $\eps$ distance between $\Hb$ and $\Ho$ to be more uniformly divided in the rows of $\Ho$.

\subsection{Proof of \cref{prop:leverage_score_sampling}}
\label{sec:proof_leverage_score_sampling}

\begin{definition}
    Given a matrix $\mathbf{A}$, the leverage score of the $i$-th row $a_i$ of $\mathbf{A}$ is defined to be $\ell_i := 
a_i(\mathbf{A}^\top \mathbf{A})^{\dagger} a_i^\top$, which is the squared $\ell_2$-norm of the $i$-th row of $\U$, where $A = \U\Sigma \mathbf{V}^T$ is the singular value decomposition of $\mathbf{A}$.
\end{definition}

It is known that sampling $O(d \log d/\eps^2)$ rows with respect to the leverage score of the matrix $[\mathbf{A}, \mathbf{b}]$ gives a $(1 \pm  \eps)$-subspace embedding of the column span of $[\mathbf{A}, \mathbf{b}]$, which means that the solution $x'$ of the regression problem $\min_x \norm{\mathbf{S}\mathbf{A} x - \mathbf{S}\mathbf{b}}_2$ satisfies  $\norm{\mathbf{A} x' - \mathbf{b}}_2 \le (1 \pm \eps) \min_x \norm{\mathbf{A} x - \mathbf{b}}_2$ (see, e.g., \cite{W14}). In the recent study of the active regression problem~\citep{CP19, MMWY22}, it turns out that sampling with respect to the leverage score of the matrix $\mathbf{A}$ itself is also sufficient to give a good solution to unknown label vector $\mathbf{b}$ with high constant probability.  

\begin{lemma}
    \label{lem:leverage_score_sampling}
    Given matrix $\mathbf{A} \in \R^{n \times d}$ and $\mathbf{b} \in \R^{n}$. Let $\mathbf{S}$ be the sampling and rescaling matrix with respect to $\ell_i(\mathbf{A})$ with $O(d \log d)$ rows. Let $x' = \mathrm{argmin}_x \norm{\mathbf{S}\mathbf{A}x - \mathbf{S}\mathbf{b}}_2$. Then we have with high constant probability, 
    \[
    \norm{\mathbf{A}x' - \mathbf{b}}_2 \le O(1) \cdot  \min_{x} \norm{\mathbf{A}x - \mathbf{b}}_2
    \]
\end{lemma}

\begin{proof}

    Since $\mathbf{S}$ is the sampling and rescaling matrix with respect to $\ell_i(\mathbf{A})$ with $O(d \log d)$ rows, we know with high constant probability, $\mathbf{S}$ is a $O(1)$-subspace embedding of $\mathbf{A}$, which means for all $x \in \R^d$, we have $\norm{\mathbf{S}\mathbf{A}x}_2 = (1 \pm \eps) \norm{\mathbf{A}x}_2$ (see, e.g., \citep{W14}).

    Now, let $x_c = \mathrm{argmin}_{x \in \R^d} \|\mathbf{S}\mathbf{A}x - \mathbf{S}\mathbf{b}\|_2$ and $x^\star = \mathrm{argmin}_{x \in \R^d} \|\mathbf{A}x - \mathbf{b}\|_2$, we have 
\[
\|\mathbf{A}x_c - \mathbf{b}\|_2 \le \|\mathbf{A}x_c - \mathbf{A}x^\star\|_2 + \|\mathbf{A}x^\star - \mathbf{b}\|_2 \le  \|\mathbf{A}x^\star - \mathbf{b}\|_2 + O(\|\mathbf{S}\mathbf{A}x_c - \mathbf{S}\mathbf{A}x^\star\|_2) \;.
\]
Also we have that 
\[
\|\mathbf{S}\mathbf{A}x_c - \mathbf{S}\mathbf{A}x^\star\|_2 \le \|\mathbf{S}\mathbf{A}x_c - \mathbf{S}\mathbf{b}\|_2 + \|\mathbf{S}\mathbf{b} - \mathbf{S}\mathbf{A}x^\star\|_2 \le 2 \|\mathbf{S}\mathbf{b} - \mathbf{S}\mathbf{A}x^\star\|_2 \;,
\]
The only remaining thing is to bound $\|\mathbf{S}\mathbf{b} - \mathbf{S}\mathbf{A}x^\star\|_2$. In fact, let $z = \mathbf{S}(\mathbf{A}x^\star - \mathbf{b})$, we have that 
\[
\mathbb{E}\left[\|\mathbf{S}(\mathbf{A}x^\star-\mathbf{b}) \|_2^2\right] = \sum_i \mathbb{E} [z_i^2] = \frac{n}{k} \sum_{i = 1}^k \sum_{j = 1}^n \frac{1}{n} (\mathbf{A}x^\star_j - \mathbf{b})^2 =\norm{\mathbf{A}x^\star - \mathbf{b}}_2^2
\]

Since we have that $\mathbb{E}\left[\|\mathbf{S}\mathbf{A}x^\star - \mathbf{S}\mathbf{b}\|_2^2\right] = \|\mathbf{A}x^\star - \mathbf{b}\|_2^2$, then by Markov's inequality we have that with high constant probability, $\|\mathbf{S}\mathbf{A}x^\star - \mathbf{S}\mathbf{b}\|_2^2 \le O(1) \|\mathbf{A}x^\star - \mathbf{b}\|_2^2$, which means that $\|\mathbf{S}\mathbf{A}x_c - \mathbf{S}\mathbf{A}x^\star\|_2 \le O(\|\mathbf{A}x^\star - \mathbf{b}\|_2)$. Put everything together and by taking a union bound, we have that with high constant probability
\[
\|\mathbf{A}x_c - \mathbf{b}\|_2 \le C \|\mathbf{A}x^\star - \mathbf{b}\|_2
\]
for some constant $C$. 
\end{proof}

\begin{proof}[Proof of the proposition]
    Since $\rank(\Ho') = d$, we can assume $\Ho' = \mathbf{A}\mathbf{B}$ where $\mathbf{A} \in \R^{D \times d}$ and $\mathbf{B} \in \R^{d \times n}$. Consider the regression problem $\min_{x} \norm{\mathbf{A}x - h_i}_2$, from $\norm{\mathbf{A}b_i - h_i}_2 = \norm{h'_i - h_i}_2 \le \eps$ we have $\min_{x} \norm{\mathbf{A}x - h_i}_2 \le \eps$.

    Let $\mathbf{S}$ be the sampling and rescaling matrix with respect  $\ell_i(\mathbf{A})$ with $O(d \log d)$ rows. From Lemma~\ref{lem:leverage_score_sampling} we have with probability at least $0.99$ we have $\min_{x} \norm{\mathbf{S}\mathbf{A}x - \mathbf{S}h_i}_2 \le C \min_{x} \norm{\mathbf{A}x - h_i}_2 \le C\eps$ and here $x^i = \mathrm{argmin}_{x} \norm{\mathbf{S}\mathbf{A}x - \mathbf{S}h_i}_2 = (\mathbf{S} \mathbf{A})^{-}\mathbf{S}h_i$. Let $I_\mathbf{S}$ denote the indices in $[n]$ where this event happens for $h_i$. Then we have $\mathbb{E}\left[|I_{\mathbf{S}}|\right] \ge 0.99n$, which means that there exists one $\mathbf{S}$ which makes $|I| \ge 0.99n$. %

    Now, taking this $\mathbf{S}$ and let $\U = \mathbf{A}(\mathbf{S} \mathbf{A}^{-})$, we have for every $i \in I_{\mathbf{S}}$, 
    \[
    \norm{\U \mathbf{S} h_i - h_i}_2 = \norm{\mathbf{A}(\mathbf{S} \mathbf{A}^{-})h_i - h_i} = \norm{\mathbf{S}\mathbf{A}x - \mathbf{S}h_i}_2 \le \bigO(\eps)\;.
    \]
    This implies the matrix $\U$ and $\mathbf{S}$ is what we need. Note that the matrix $\mathbf{S}$ is a diagonal matrix but each diagonal entry has a rescaling weight, but we can put the weights into the row of $\U$ and make $\mathbf{S}$ as a row selection matrix (where each non-zero entry has value $1$).

\end{proof}

\subsection{Proof of \cref{thrm:narrow_network_cluster}}
\label{sec:proof_narrow_network_cluster}

\begin{lemma}
\label{lem:attention_pool_closeness}
    If in an attention mechanism we have $\norm{v_i}\leq \eta$ for some $\eta=\bigO(1)$, $\exp(-c_1 \eps) a_{ij}  \leq \ah_{ij} \leq \exp(c_1 \eps) a_{ij}$, $\forall i,j: \abs{ v_i\cdot v_j - \vh_i\cdot \vh_j} < c_2\eps$, for some $c_1, c_2 = \bigO(\eps)$, and we have $h_i = \sum_{u \in \mathcal{N}(i)}a_{iu}v_u$ and $\hh_i = \sum_{u \in \mathcal{N}(i)}\ah_{iu}\vh_u$, then there is some $t\in \bigO(1)$ that $\abs{ h_i\cdot h_j - \hh_i\cdot \hh_j} < t\eps$.
\end{lemma}
\begin{proof}
    For any pair $i, j$, we have:

    \begin{align*}
        \hh_i\cdot \hh_j &= \sum_{u \in \mathcal{N}(i)} \ah_{iu} \vh_u \sum_{\nu \in \mathcal{N}(j)} \ah_{j\nu} \vh_{\nu}\\
        &= \sum_{u, \nu \in \mathcal{N}(i), \mathcal{N}(j)} \ah_{iu} \ah_{j\nu} (\vh_u \cdot\vh_{\nu})\\
        & < \sum_{u, \nu \in \mathcal{N}(i), \mathcal{N}(j)} \exp(2c_1\eps) a_{iu} a_{j\nu} (v_u \cdot v_{\nu} + c_2\eps)\\
        &= \exp(2c_1\eps) \sum_{u, \nu \in \mathcal{N}(i), \mathcal{N}(j)} a_{iu} a_{j\nu} (v_u \cdot v_{\nu} + c_2\eps)\\
        &= \exp(2c_1\eps) \left(\left(\sum_{u, \nu \in \mathcal{N}(i), \mathcal{N}(j)} a_{iu} a_{j\nu} v_u \cdot v_{\nu}\right) + c_2\eps \left(\sum_{u, \nu \in \mathcal{N}(i), \mathcal{N}(j)} a_{iu} a_{j\nu} \right)\right)\\
        &=  \exp(2c_1\eps) \left(\left(\sum_{u, \nu \in \mathcal{N}(i), \mathcal{N}(j)} a_{iu} a_{j\nu} v_u \cdot v_{\nu}\right) + c_2\eps \left(\sum_{u \in \mathcal{N}(i)} a_{iu} \sum_{\nu \in \mathcal{N}(j)}a_{j\nu} \right)\right)\\
        &=  \exp(2c_1\eps) \left(\left(\sum_{u, \nu \in \mathcal{N}(i), \mathcal{N}(j)} a_{iu} a_{j\nu} v_u \cdot v_{\nu}\right) + c_2\eps  \right)\\
        &=  \exp(2c_1\eps) \left(h_i\cdot h_j + c_2\eps\right)\\
    \end{align*}

    Now, similarly, by lower bounding the attention scores and dot products of the $\V$ vectors from the smaller network, we will have:
    \[
    \hh_i\cdot \hh_j > \exp(-2c_1\eps) \left(h_i\cdot h_j - c_2\eps\right)
    \]
    Since $1+x < \exp(x) < 1+2x $ for $0 < x \leq 1$ and $2c_1\eps < 1$, and, 
    
    \begin{align*}
        h_i\cdot h_j &= \sum_{u, \nu \in \mathcal{N}(i), \mathcal{N}(j)} a_{iu} a_{j\nu} (v_u \cdot v_{\nu})\\
        &\leq \sum_{u, \nu \in \mathcal{N}(i), \mathcal{N}(j)} a_{iu} a_{j\nu} \eta \\
        & = \eta,
    \end{align*}
    
    If $h_i\cdot h_j + \alpha\eps > 0$,
    \begin{align*}
        \hh_i\cdot \hh_j &< (1+8\alpha\eps) \left(h_i\cdot h_j + \alpha\eps\right)\\
        &=  h_i\cdot h_j + \eps \left(8\alpha(h_i\cdot h_j) + 8\alpha^2\eps + \alpha \right)\\
        & \leq h_i\cdot h_j + \eps \left(8\alpha^2 + \alpha + 8\alpha^2\eps  \right)
    \end{align*}
    Otherwise,
    \begin{align*}
        \hh_i\cdot \hh_j &< (1+4\alpha\eps) \left(h_i\cdot h_j + \alpha\eps\right)\\
        &=  h_i\cdot h_j + \eps \left(4\alpha^2 + 4\alpha^2\eps + \alpha \right).
    \end{align*}
    
    For the lower bound we have $1-x \leq \exp(-x) \leq 1-x/2$ for $0 < x < 1$, and,
    \begin{align*}
        h_i\cdot h_j &= \sum_{u, \nu \in \mathcal{N}(i), \mathcal{N}(j)} a_{iu} a_{j\nu} (v_u \cdot v_{\nu})\\
        &\geq \sum_{u, \nu \in \mathcal{N}(i), \mathcal{N}(j)} a_{iu} a_{j\nu} -\alpha \\
        & = -\alpha,
    \end{align*}
    
    Now, if $h_i\cdot h_j - \alpha\eps > 0$,
    \begin{align*}
        \hh_i\cdot \hh_j &> (1-2c_1\eps) \left(h_i\cdot h_j - \alpha\eps\right)\\
        &=  h_i\cdot h_j - \eps \left(-4\alpha(h_i\cdot h_j) + 4\alpha^2\eps - \alpha \right)\\
        & > h_i\cdot h_j - \eps \left(4\alpha^2 + 4\alpha^2\eps - \alpha \right)
    \end{align*}

    Otherwise:

    \begin{align*}
        \hh_i\cdot \hh_j &> (1-2\eps\alpha) \left(h_i\cdot h_j - \alpha\eps\right)\\
        &=  h_i\cdot h_j - \eps \left(-2\alpha(h_i\cdot h_j) + 2\alpha^2\eps - \alpha \right)\\
        & > h_i\cdot h_j - \eps \left(2\alpha^2 + 2\alpha^2\eps - \alpha \right)
    \end{align*}

    since $\alpha$ is a positive value, if we get $t = 8\alpha^2 + \alpha + 8\alpha^2\eps$, we have the proof for the lemma.
\end{proof}

\begin{lemma}
\label{lem:attention_layer_JL}
    Assume for each $i$, $\norm{q_i^{(1)}}, \norm{k_i^{(1)}}, \norm{v_i^{(1)}} \leq \sqrt{\alpha}$, and $d = \mathcal{O}(\frac{\log n}{\eps^2})$ for any $0 < \eps < \frac{1}{8\alpha}$. For the first Transformer layer's attention, there exist $\Qh^{(1)}, \Kh^{(1)},\Vh^{(1)} \in \R^{d \times D}$, such that for any pair $i, j$, $h_i^{(1/2)}\cdot h_j^{(1/2)} - t\eps \leq \widehat{h}_i^{(1/2)}\cdot \widehat{h}_j^{(1/2)} \leq h_i^{(1/2)} \cdot h_j^{(1/2)} + t\eps$ for a constant $t = \bigO(1)$.
\end{lemma}

\begin{proof}
    As part of proof for Theorem~\ref{thrm:narrow_attention}, we saw that for the narrow network in the first layer we have $\exp(-2\alpha \eps) a_{ij}^{(1)}  \leq \ah_{ij}^{(1)} \leq \exp(2\alpha \eps) a_{ij}^{(1)}$. Now, if we consider $n$ vectors from $\V^{(1)}$, by the JL-Transform, we know there exists a linear map $f_V$, with weight matrix $\mathbf{M}_V \in \R^{d \times D}$, such that:
    $$\forall i,j: \abs{v_i^{(1)}\cdot v_j^{(1)} - f(v_i^{(1)})\cdot f(v_j^{(1)})} < \alpha\eps.$$

    Now, if we consider $\wh{}{1} = \mathbf{M}_V \w{}{1}$, we will have:
    $$\forall i,j: \abs{v_i^{(1)}\cdot v_j^{(1)} - \vh_i^{(1)}\cdot \vh_j^{(1)}} < \alpha\eps.$$
    Now, since the conditions of the \cref{lem:attention_pool_closeness} are correct here, we have the proof of the lemma. 
\end{proof}

\begin{proof}[Proof of the Theorem]
If $d = \bigO(\frac{\log n}{\eps^2})$ we saw from \cref{lem:attention_layer_JL} that we can have $\Q^{(1)}, \K^{(1)}$, and $\V^{(1)}$ all in $\R^{d \times n}$ in a way that for each $i, j$ we have $\abs{\hh^{1/2}_i \cdot \hh^{1/2}_j - h^{1/2}_i \cdot h^{1/2}_j} < t\eps$ for some constant $t = \bigO(1)$.

If input matrix $\X$ has a low-rank approximation $\Bar{X}$ such that $\rank(\Bar{X}) \leq d$ and for each $i$, $\norm{X_i - \Bar{X}_i} \leq \eps$, we will prove the following statement first:

\begin{lemma}
\label{lem:dot_product_low_rank}
    If $\Hh\llo \in \R^{d \times n}$, is the input vector of a narrow Transformer layer in a way that there exists a $\U\llo \in \R^{D \times d}$ that $\forall i: \norm{\U\llo\hh\llo_i - h\llo_i} \leq \eps\llo$ for a sufficiently small $\eps\llo$, we can have $\Wh_Q, \Wh_K$, and $\Wh_V \in \R^{d \times d}$ that the following conditions hold:
    \begin{enumerate}
        \item $\forall i, j: h\llh_i\cdot h\llh_j - t\eps\llo \leq \hh_i\llh \cdot \hh_j\llh \leq h_i\llh \cdot h_j\llh + t\eps\llo$ for some $t = \bigO(1)$.
        \item For any attention score $\frac{a_{ij}\llo}{\ah_{ij}\llo} = 1 + \mathcal{O}(\eps\llo)$.
    \end{enumerate}
\end{lemma}

\begin{proof}[Proof of \cref{lem:dot_product_low_rank}:]
In this proof, we will occasionally emit the $\ell$ superscripts, since all the vectors are in the same layer.

According to \cref{thrm:narrow_network_low_rank_2}, we can construct $\Wh_Q$ and $\Wh_K \in \R^{d \times d}$ that:
$\exp(-2\eps_a) \leq \frac{\ah_{ij}}{a_{ij}} \leq \exp(2\eps_a)$, where $\eps_a = \beta(2\sqrt{\alpha}+\beta)\eps\llo$. And for a small enough $\eps\llo$ we have $2\eps_a < 1$, thus $\exp(2\eps_a) < 1 + 2\eps_a$ and $\exp(-2\eps_a) > 1 - \eps_a$, and for any $i, j$, $a\llo_{i,j} \leq 1$. Thus for any $i$ and $j$, $\frac{a\llo_{i,j}}{\ah\llo_{i,j}} = 1 + \bigO(\eps_a) = 1 + \bigO(\eps\llo)$. 

For the $\Wh_V$, we want to make it in a way that $\forall i, j: \abs{\vh_i\cdot\vh_j - v_i \cdot v_j} = \bigO(\eps\llo)$.

We have $v_i \cdot v_j = v_i\tp v_j = h_i\tp\W_V\tp\W_V h_j$. And also:

\begin{align*}
    &\abs{\hh_i\tp\U\tp\W_V\tp\W_V\U\hh_j - h_i\tp\W_V\tp\W_V h_j} = \abs{(\U\hh_i)\tp\W_V\tp\W_V (\U\hh_j) - h_i\tp\W_V\tp\W_V h_j}\\
    &= \abs{(\U\hh_i)\tp\W_V\tp\W_V (\U\hh_j) - (\U\hh_i)\tp\W_V\tp\W_V h_j + (\U\hh_i)\tp\W_V\tp\W_V h_j - h_i\tp\W_V\tp\W_V h_j} \\
    &= \abs{(\U\hh_i)\tp\W_V\tp\W_V (\U\hh_j - h_j) + (\U\hh_i - h_i)\tp\W_V\tp\W_V h_j} \\
    &\leq \abs{(\U\hh_i)\tp\W_V\tp\W_V (\U\hh_j - h_j)} + \abs{(\U\hh_i - h_i)\tp\W_V\tp\W_V h_j} \\
    &\leq (\sqrt{\alpha} + \eps\llo)\beta^2 \eps\llo + \eps\llo \beta^2\sqrt{\alpha} \\
    & \leq (2\beta^2\sqrt{\alpha} + \beta^2) \eps\llo.
\end{align*}
Now, if we take $\Wh_V = \U\tp\W_V\tp\W_V\U$, we have $\abs{\vh_i\cdot\vh_j - v_i \cdot v_j} \leq (2\beta^2\sqrt{\alpha} + \beta^2) \eps\llo$. For the brevity we define $\eps_v = (2\beta^2\sqrt{\alpha} + \beta^2) \eps\llo$.

   With similar calculations as \cref{lem:attention_layer_JL}, we have:
     $$\forall i, j: h\llh_i\cdot h\llh_j - t\eps\llo \leq \hh_i\llh \cdot \hh_j\llh \leq h_i\llh \cdot h_j\llh + t\eps\llo$$ for some $t \in \bigO(1)$.
\end{proof}

We will inductively show that the initial assumption described in that theorem holds, and thus the construction of those mappings are feasible.

In layer $\ell$ for each $i$ and $j$, assume we have $\abs{\hh\llh_i\cdot \hh\llh_i - h\llh_i \cdot h\llh_j} \leq t\llo\eps$ for some constant $t\llo$. The base of this induction works because of the construction and proves in \cref{lem:attention_layer_JL} and \cref{thrm:narrow_network_low_rank_2}.

Based on the clustering assumptions we prove the following lemmas that will help us in rest of the proof:

\begin{lemma}
\label{lem:same_cluster}
    If $h_i$ and $h_j$ are in the same cluster with cluster center $c_a$, we have $\frac{h_i \cdot h_j}{\norm{c_a}^2} = 1 + \bigO(\eps)$.
\end{lemma}
\begin{proof}[Proof of \cref{lem:same_cluster}:]
    \begin{align*}
        h_i \cdot h_j &= ((h_i - c_a) + c_a) \cdot ((h_j - c_a) + c_a) \\
        &\leq \norm{h_i - c_a}\norm{c_a} + \norm{h_j - c_a}\norm{c_a} + \norm{h_i - c_a} \norm{h_j - c_a} + \norm{c_a}^2 \\
        &\leq 2\eps \norm{c_a}^2 + \eps^2 \norm{c_a}^2 + \norm{c_a}^2 \\
        &= (1 + 2\eps + \eps^2) \norm{c_a}^2 \leq (1 + 3\eps) \norm{c_a}^2.
    \end{align*}

    Also, similarly we have: 

    \begin{align*}
        h_i \cdot h_j &= ((h_i - c_a) + c_a) \cdot ((h_j - c_a) + c_a) \\
        &\leq -\norm{h_i - c_a}\norm{c_a} - \norm{h_j - c_a}\norm{c_a} - \norm{h_i - c_a} \norm{h_j - c_a} + \norm{c_a}^2 \\
        &\leq -2\eps \norm{c_a}^2 - \eps^2 \norm{c_a}^2 + \norm{c_a}^2 \\
        &= (1 - 2\eps - \eps^2) \norm{c_a}^2 \leq (1 - 3\eps) \norm{c_a}^2.
    \end{align*}

    Thus, we have $1 - 3\eps \leq \frac{h_i \cdot h_j}{\norm{c_a}^2} \leq 1 + 3\eps $. \end{proof}

\begin{corollary}
\label{corollary_same_cluster}
    If $h_i$ and $h_j$ are in the same cluster with cluster center $c_a$, we have $\frac{\hh_i \cdot \hh_j}{\norm{c_a}^2} = 1 + \bigO(\eps)$.
\end{corollary}

\begin{proof}[Proof of \cref{corollary_same_cluster}:]
    We saw that $\frac{h_i \cdot h_j}{\norm{c_a}^2} = 1 + \bigO(\eps)$ and we know $\abs{\hh_i \cdot \hh_j - h_i\cdot h_j} < t\llo\eps$. Since $\norm{c_a}^2 > \gamma_1^2$ the corollary is correct.
\end{proof}

\begin{lemma}
\label{lem:different_cluster}
    If $h_i$ and $h_j$ are from different clusters with cluster centers $c_a$ and $c_b$ accordingly, we have $\frac{h_i \cdot h_j}{\norm{c_a}^2} \leq 0.5 + \bigO(\eps)$.
\end{lemma}
\begin{proof}[Proof of \cref{lem:different_cluster}:]
    \begin{align*}
        h_i \cdot h_j &= ((h_i - c_a) + c_a) \cdot ((h_j - c_b) + c_b) \\
        &\leq \norm{h_i - c_a}\norm{c_b} + \norm{h_j - c_b}\norm{c_a} + \norm{h_i - c_a} \norm{h_j - c_b} + c_a \cdot c_b \\
        &\leq 2\eps \norm{c_a}\norm{c_b} + \eps^2 \norm{c_a}\norm{c_b} + \gamma_1^2/2 \\
        &= (2\eps + \eps^2) \gamma_2^2 + 1/2 \gamma_1^2 \leq 1/2 \norm{c_a}^2 + 3\eps \gamma_2^2.
    \end{align*}

    Thus we have $\frac{h_i \cdot h_j}{\norm{c_a}^2} \leq 0.5 + 3\eps \gamma_2^2/\norm{c_a}^2 \leq  0.5 + 3\eps \gamma_2^2/\gamma_1^2 = 0.5 + \bigO(\eps)$. \end{proof}

\begin{corollary}
\label{corollary:different_cluster}
    If $h_i$ and $h_j$ are from different clusters with cluster centers $c_a$ and $c_b$ accordingly, we have $\frac{\hh_i \cdot \hh_j}{\norm{c_a}^2} \leq 0.5 + \bigO(\eps)$.
\end{corollary}

\begin{proof}[Proof of \cref{corollary:different_cluster}:]

    Again this is an immediate result of the previous lemma and $$\abs{\hh_i \cdot \hh_j - h_i\cdot h_j} < t\llo \eps.$$ \end{proof}

Now, we construct the feed forward layers following the attention layer: for constructing $\Wh\llo_1$, as we do not have the centers of the clusters in the low dimension particularly, we select just one of the node representations from that cluster as $\hat{c}_a$. If the number of clusters is smaller than $d$ we set cluster centers for these clusters to all zero vectors. We set the $a$th row of $\Wh\llo_1$ to be $4\frac{\hat{c}_a}{\norm{c_a}^2}$, we use bias of the linear mappings here and set it to $-3\norm{c_a}^2$.

Now, for each $\hh_i$ if it is in cluster $a$, we have $(\Wh\llo_1\hh_i)_a = 1 + \bigO(\eps)$ due to \cref{corollary_same_cluster} and for a small enough $\eps$, we have $\forall b\neq a: (\Wh\llo_1\hh_i)_b < 0$. Thus $ \ReLU(\Wh\llo_1\hh_i)$ is almost one-hot vector, being absolutely zero in any index $b\neq a$ and have value $1 + \bigO(\eps)$ in index $a$.

We will use these almost 1-hot vectors to create the embeddings using $\Wh\llo_2$. But before doing that, we will show that the assumption will lead to having a low-rank approximation on the higher dimension network after the $\ReLU$:

\begin{lemma}
    $\Ho\llt$ has a low-rank estimation $\Hb\llt$ of maximum rank $d$ such that $\norm{h\llt_i - \hb\llt_i} = \bigO(\eps)$.
\end{lemma}

\begin{proof}
    Since for each $h\llh_i$ we have a $c_a$ such that their distance is $\bigO(\eps)$, and due to the maximum operator norm of $\W_1\llo$ and Lipschitzness of $\ReLU$, we have $\norm{\ReLU(W_1\llo h_i) - \ReLU(W_1\llo c_a)} = \bigO(\eps)$. Now we can make $\Hb\llt$ in this way that instead of each $h\llt_i$ replace it with the mapping of its cluster center $\ReLU(W_1\llo c_a)$. Now, $\Hb\llt$ has maximum $d$ distinct rows, thus its rank is maximum $d$.
\end{proof}

Since $\Hb$ is of rank maximum $d$, we can have $\U \in \R^{D \times d}$ and $\Lam \in {d \times D}$ that $\Hb = \U\Lam\Hb$.
Now, for creating $\Wh\llo_2$ we make a $d \times d$ matrix having column $a$ as $\Lam \W\llo_2 \ReLU(\W\llo_1 c_a)$. Now, each row of $\Hh\lln$ will be a $(1+\bigO(\eps)) \W\llo_2 \ReLU (\W_1c_a)$ for some $c_a$, thus each column of $U\Hh\lln$ will estimate the corresponding column of $\Ho\lln$ with maximum norm-2 distance of $\bigO(\eps)$.

This will satisfy the required conditions for the start of a layer as described in \cref{thrm:narrow_network_low_rank_2}. Thus, we can create low-rank matrices  
\end{proof}

\section{Experiment Details}
\label{sec:exp_details}
\subsection{Datasets}

Below, we provide descriptions of the datasets on which we conduct experiments. Summary statistics of these datasets are provided in \cref{table:datasets}.

\paragraph{Photo}
This dataset is part of the Amazon co-purchase graph \citep{mcauley2015image}, where each node represents a product on the Amazon website, and an edge indicates that the corresponding products were frequently purchased together. Node features are derived from a bag-of-words summary of the reviews for each product. The task is to classify the nodes into different product categories \citep{shchur2018pitfalls}. We use a random train/validation/test split with a ratio of 0.6/0.2/0.2 for training.

\paragraph{Minesweeper}
This dataset was first introduced in \cite{platonov2023critical}. The dataset is a graph representation of the 100x100 grid from the Minesweeper game. Each node represents a cell, and edges connect to the eight neighboring cells. 20\% of the nodes are marked as mines. The node features are the one-hot encoding of the number of mines among the neighbors. For 50\% of the nodes, the features are unknown and indicated by a separate binary feature.

\paragraph{Tolokers}
This dataset was also first introduced in the \cite{platonov2023critical}. Tolokers is a graph representation of workers on a crowdsourcing platform called Toloka. Two nodes are connected if the workers have worked on the same project. Node features are based on the worker's task performance statistics and other profile information. The task is to predict which nodes have been banned for a project.

\begin{table}[htp]
    \centering
    \caption{Dataset statistics. The reported number of edges is the number of directed edges, which will be twice the number of actual edges for the undirected graphs. The homophily score is a metric to measure what ratio of the neighbor nodes are from the same class.}
    \fontsize{8.5pt}{8.5pt}\selectfont
    \setlength\tabcolsep{6pt} %
    \scalebox{1}{
    \begin{tabular}{lcccccc}
    \toprule
         {\bf Dataset} & {\bf Nodes} & {\bf Edges} & {\bf Homophily Score} & {\bf Node Features} & {\bf Classes} &{\bf Metric} \\
         \midrule
Amazon Photo & 7,487 & 238,162 & 0.772 & 745 & 8 & Accuracy \\
Minesweeper & 10,000 & 78,804 & 0.009 & 7 & 2 & AUC\\ 
Tolokers & 11,758 & 1,038,000 & 0.187 & 10 & 10 & AUC \\
    \bottomrule
    \end{tabular}
    }
     \label{table:datasets}
\end{table}

\subsection{Network Details}
We train the Transformer model only on the graph edges here. Another alternative could be to train it with an augmented expander graph as suggested in \cite{shirzad2023exphormer}. The model implementation is very similar to the formulation explained in \cref{sec:notations}; however, for performance purposes, there are a few small changes:
\begin{enumerate}
    \item We follow the standard attention implementation and add $\frac{1}{\sqrt{D}}$ normalization to the attention scores. This works much better in practice.
    \item We use skip connections to improve the training process.
    \item We use batch normalization instead of layer normalization, as our experiments and the results from the Exphormer model showed that batch normalization actually works better than layer normalization in Graph Transformers.
\end{enumerate}

For all datasets and for both large and small models we used four layers of the network.

\subsection{Hyperparameter Search} For both large and small models, we performed a grid search on the base learning rate and the number of epochs, selecting the configuration with the highest average accuracy or AUC based on the standard metric for the dataset. We chose the base learning rate from $\{0.1, 0.01, 0.001\}$ and the number of epochs from $\{50, 100, 150, 200\}$. We used AdamW \citep{loshchilov2017fixing} for optimization and a cosine learning rate scheduler \citep{loshchilov2016sgdr}, as is common in training Transformers and Graph Transformers. We did twenty runs for larger networks to measure the mean and 100 runs for the small networks.

\subsection{Operator Norms and Vector Norms}
\label{subsec:norms}
To validate the assumptions in \cref{sec:notations}, we measured the average operator norm of the linear mappings and the norm of the vectors for the input of each layer, and then averaged these values across the layers in three datasets we experimented on. Results are provided in \cref{tab:norms}. The resulting numbers are reasonably small, allowing us to assume them to be $\bigO(1)$ in the theory.

\subsection{Comparing the Results to Baselines}
To understand the significance of the results and to provide context, we compare the performance of our compressed network with simple baselines, namely MLP, GCN, Nodeformer, and Exphormer models. Results are provided in \cref{tab:results_comp}. Our goal is to demonstrate that a very small network with a hidden dimension of 4 can achieve much better results than these methods, given that we can properly compress the large network.

\begin{table}[h!]
    \centering
    \caption{Comparing found results with some baselines, showing the models are well-trained and have competitive results.}
    \begin{tabular}{l|ccc}
    \toprule
        Dataset & Tolokers & Minesweeper & Photo\\
        \midrule
        MLP & $0.730 \pm 0.0106$ & $ 0.509 \pm 0.014$ & $0.696 \pm 0.038$\\
        GCN & $0.836 \pm 0.007$ & $0.898 \pm 0.005$ & $ 0.927 \pm 0.002$\\
        NodeFormer & $0.781 \pm 0.001$ & $0.867 \pm 0.009$ & $0.935 \pm 0.004$\\
        Exphormer & $0.835 \pm 0.003$ & $0.923 \pm 0.006$ & $0.954 \pm 0.002$ \\
        \midrule
        Average Large Network & $0.844 \pm 0.002$ & $0.943 \pm 0.001$ & $0.953 \pm 0.004$\\
        Average Small Network & $0.821 \pm 0.011$ & $0.886 \pm 0.054$ & $0.910 \pm 0.016$\\
        Max Small Network & $0.844$ & $0.938$ & $0.944$\\
    \bottomrule
    \end{tabular}
    
    \label{tab:results_comp}
\end{table}
\end{document}